\documentclass{article}

\usepackage{microtype}
\usepackage{graphicx}
\usepackage{subfigure}
\usepackage{booktabs} 
\usepackage{xcolor}

\usepackage[pagebackref=true,breaklinks=true,colorlinks,bookmarks=false]{hyperref}
 \hypersetup{
 linkcolor=red,
 filecolor=red,
 citecolor=green,      
 urlcolor=cyan,
}

\usepackage{pgfplots}
\pgfplotsset{compat = newest}
\usepackage{multirow}

\usepackage{framed}
\usepackage{tcolorbox}

\usepackage{amsmath}
\usepackage{amssymb}
\usepackage{mathrsfs}
\usepackage{mathtools}
\usepackage{amsthm}
\usepackage[ruled,vlined]{algorithm2e}
\usepackage{algorithmic}
\usepackage{listings}
\usepackage{makecell}
\usepackage{colortbl}
\usepackage{color}
\usepackage{wrapfig}
\usepackage{cancel}
\usepackage{soul,xcolor}
\usepackage{pifont}
\usepackage{tikz}
\usetikzlibrary{shapes, positioning, arrows.meta, calc, decorations.pathmorphing, quotes}

\usepackage[capitalize,noabbrev]{cleveref}

\theoremstyle{plain}
\newtheorem{theorem}{Theorem}[section]
\newtheorem{proposition}[theorem]{Proposition}
\newtheorem{lemma}[theorem]{Lemma}

\theoremstyle{definition}
\newtheorem{definition}[theorem]{Definition}

\theoremstyle{remark}

\newcommand{\alink}[1]{\href{#1}{paper-link}}

\newcommand{\bfC}{\mathbf{C}}
\newcommand{\bfH}{\mathbf{H}}
\newcommand{\bfI}{\mathbf{I}}

\newcommand{\bfZ}{\mathbf{Z}}

\usepackage{enumitem}
\usepackage{ amssymb }

\definecolor{citecolor}{HTML}{0071BC}
\definecolor{linkcolor}{HTML}{ED1C24}
\hypersetup{colorlinks=true, linkcolor=linkcolor, citecolor=citecolor,urlcolor=black}

\definecolor{commentcolor}{RGB}{110,154,155}   
\newcommand{\PyComment}[1]{\footnotesize\ttfamily\textcolor{commentcolor}{\# #1}}  
\newcommand{\PyCode}[1]{\footnotesize\ttfamily\textcolor{black}{#1}} 

\usepackage{amsmath,amsfonts,bm}

\def\eqref#1{equation~\ref{#1}}

\def\1{\bm{1}}

\DeclareMathAlphabet{\mathsfit}{\encodingdefault}{\sfdefault}{m}{sl}
\SetMathAlphabet{\mathsfit}{bold}{\encodingdefault}{\sfdefault}{bx}{n}

\newcommand{\R}{\mathbb{R}}

\usepackage{hyperref}

\usepackage[accepted]{icml2024}

\usepackage[capitalize,noabbrev]{cleveref} 

\icmltitlerunning{Matrix Information Theory for Self-Supervised Learning}

\begin{document}

\twocolumn[
\icmltitle{Matrix Information Theory for Self-Supervised Learning}



\icmlsetsymbol{equal}{*}

\begin{icmlauthorlist}
\icmlauthor{Yifan Zhang}{equal,iiis}
\icmlauthor{Zhiquan Tan}{equal,thumath}
\icmlauthor{Jingqin Yang}{equal,iiis}
\icmlauthor{Weiran Huang}{sjtu,shailab}
\icmlauthor{Yang Yuan}{iiis,shailab,qizhi}
\end{icmlauthorlist}

\icmlaffiliation{iiis}{IIIS, Tsinghua University, Beijing, China}
\icmlaffiliation{thumath}{Department of Mathematical Sciences, Tsinghua University, Beijing, China}
\icmlaffiliation{shailab}{Shanghai AI Laboratory, Shanghai, China}
\icmlaffiliation{qizhi}{Shanghai Qizhi Institute, Shanghai, China}
\icmlaffiliation{sjtu}{MIFA Lab, Qing Yuan Research Institute, SEIEE, Shanghai Jiao Tong University, Shanghai, China}

\icmlcorrespondingauthor{Yang Yuan}{yuanyang@tsinghua.edu.cn}
\icmlkeywords{Machine Learning, Self-Supervised Learning, ICML}
\vskip 0.3in
]



\printAffiliationsAndNotice{\icmlEqualContribution} 


\begin{abstract}
The maximum entropy encoding framework provides a unified perspective for many non-contrastive learning methods like SimSiam, Barlow Twins, and MEC. 
Inspired by this framework, we introduce Matrix-SSL, a novel approach that leverages matrix information theory to interpret the maximum entropy encoding loss as matrix uniformity loss. Furthermore, Matrix-SSL enhances the maximum entropy encoding method by seamlessly incorporating matrix alignment loss, directly aligning covariance matrices in different branches.
Experimental results reveal that Matrix-SSL outperforms state-of-the-art methods on the ImageNet dataset under linear evaluation settings and on MS-COCO for transfer learning tasks. Specifically, when performing transfer learning tasks on MS-COCO, our method outperforms previous SOTA methods such as MoCo v2 and BYOL up to 3.3\% with only 400 epochs compared to 800 epochs pre-training. We also try to introduce representation learning into the language modeling regime by fine-tuning a 7B model using matrix cross-entropy loss, with a margin of 3.1\% on the GSM8K dataset over the standard cross-entropy loss. 
\end{abstract}

\section{Introduction}
Contrastive learning methods \citep{chen2020simple,he2020momentum} focus on aligning similar objects closely while distancing dissimilar ones. This approach, grounded in intuitive principles, has led to deep and interesting insights. For example, 
SimCLR has been proved to perform spectral clustering on similarity graph~\citep{tan2023contrastive,haochen2021provable}, and \citet{wang2020understanding} highlight two critical aspects of contrastive loss: \textbf{alignment and uniformity}.

Alignment loss ensures similar objects are closely mapped,  whereas uniformity loss promotes a uniformly distributed output feature space 
that preserves the maximum information. 
Remarkably, many existing contrastive methods~\citep{wu2018unsupervised,he2020momentum,logeswaran2018efficient,
tian2020contrastive,hjelm2018learning,bachman2019learning,chen2020simple} can be viewed as specific implementations of these two loss types, a perspective that simplifies understanding their core mechanisms.

Simultaneously, there is growing interest in non-contrastive learning methods 
that do not use negative samples, such as
BYOL~\citep{grill2020bootstrap}, SimSiam~\citep{chen2021exploring}, Barlow Twins~\citep{zbontar2021barlow}, VICReg~\citep{bardes2021vicreg}, etc. 
Among these, 
\citet{liu2022self} presented an interesting theoretical framework called maximum entropy encoding, which proposes to maximize the following loss between the two feature matrices $\mathbf{Z}_1,\mathbf{Z}_2$ computed from different augmentations from the same input:
\[
 \mathcal{L}_{\text{MEC}} = -\mu \log \operatorname{det}\left(\mathbf{I}_{d}+\lambda \mathbf{Z}_1 \mathbf{Z}_2^\top\right).
\]
Although it may not be immediately obvious,  the above loss encourages maximum entropy encoding for the feature embeddings, which is similar to the \textbf{uniformity loss} in contrastive learning methods.
It turns out that this formulation naturally encompasses loss functions of several other non-contrastive methods like SimSiam, Barlow Twins, 
and the resulting algorithm MEC surpasses previous methods in performance~\citep{liu2022self} (element-wise alignment losses such as $\|\mathbf{z}_1 - \mathbf{z}_2\|_2$ used in BYOL can be seen as low-order Taylor expansion terms in this MEC loss).
However, a comparison of contrastive and non-contrastive methods reveals some differences:
\begin{equation*}
\begin{tabular}{c|l}
    \toprule
    \textbf{Learning Method} & \textbf{Loss Function} \\
    \midrule
    Contrastive Learning & Uniformity + Alignment \\
    Non-contrastive Learning & Uniformity \\
    \bottomrule
\end{tabular}
\end{equation*}

This observation naturally propels us towards a broader, more explorative query:

\begin{figure*}[ht]
    \centering
\input{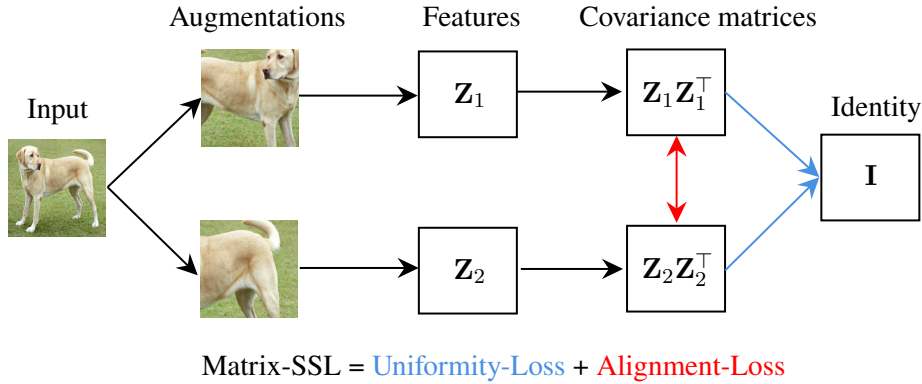}
    \caption{Illustration of the Matrix-SSL architecture. The diagram begins with the image input layer, followed by data augmentations and feature extraction, leading to the formation of covariance matrices (\(\mathbf{Z}_1\mathbf{Z}_1^\top\) and \(\mathbf{Z}_2\mathbf{Z}_2^\top\)).}
    \label{fig:arch}
\end{figure*}

\begin{center}
\fbox{\parbox{0.8\columnwidth}{
\centering
\textbf{
Could there exist a more encompassing framework
that harmonizes the virtues of both contrastive and non-contrastive learning methods?
}}}
\end{center}

In this paper, we affirmatively address this question, presenting a method that not only integrates but also enhances the advantages of both contrastive and non-contrastive learning paradigms.

The existing maximum entropy encoding framework, however, does not explicitly differentiate between feature matrices from different branches, hindering its integration with alignment loss. To bridge this gap, we introduce matrix information theory. By extending classical concepts like entropy, Kullback–Leibler (KL) divergence, and cross-entropy to matrix analogs, we offer a richer representation of associated loss functions. Notably, we find that methods like SimSiam, BYOL, Barlow Twins, and MEC can be reinterpreted as utilizing matrix cross-entropy (MCE)-based loss functions, a connection previously unexplored (see Theorem~\ref{thm:mce-tcr-md}).

Our proposed algorithm, Matrix-SSL, incorporates matrix alignment loss into non-contrastive methods, leading to improvements in empirical performance. This dual focus provides additional information and a richer signal for representation learning.
Matrix-SSL includes Matrix-Uniformity and Matrix-Alignment loss components. 
Matrix-Uniformity aligns the cross-covariance matrix of feature matrices $\bfZ_1$
and $\bfZ_2$ with the identity matrix $\bfI_d$, while Matrix-Alignment focuses on aligning their auto-covariance matrices (see Figure~\ref{fig:arch}). As a by-product, we observe the closed-form relationship between effective rank and matrix KL, which indicates that effective rank can be a powerful metric for measuring performance for various machine learning methods (see Section~\ref{sec:effective-rank}).

In experimental evaluations, our method Matrix-SSL outperforms state-of-the-art methods (SimCLR, BYOL, SimSiam, Barlow Twins, VICReg, etc.) on ImageNet datasets, especially under linear evaluation settings, our method uses only 100 epochs pre-training can outperform SimCLR 100 epochs pre-training by 4.6\%. For transfer learning tasks such as COCO detection and COCO instance segmentation, our method outperforms previous SOTA methods such as MoCo v2 and BYOL up to 3\% with only 400 epochs compared to 800 epochs pre-training. 

We further introduce representation learning into the language modeling regime and use the matrix cross-entropy loss to fine-tune large language models, achieving SOTA results on the GSM8K dataset for mathematical reasoning with a margin of 3.1\% over standard cross-entropy loss. 

In summary, our contributions can be listed as three-fold:

\vspace{-0.2cm}
\begin{itemize}[leftmargin=0.5cm, parsep=0pt, itemsep=0pt, topsep=-0.5pt]
    \item We prove the equivalence of MEC loss and matrix uniformity loss (up to constant terms and factors) in non-contrastive learning, and the closed-form relationship between effective rank and matrix KL.

    \item We provide a unified perspective of uniformity loss plus alignment loss for both contrastive and non-contrastive learning methods. 

    \item We empirically verify our method under various tasks including linear evaluation on image classification tasks, transfer learning on object detection and instance segmentation tasks, and large language model fine-tuning for mathematical reasoning tasks.  
\end{itemize}

\section{Related Work}

\textbf{Contrastive and non-contrastive SSL approaches.\hspace{1ex}} Contrastive and non-contrastive self-supervised learning methods learn representations based on diverse views or augmentations of inputs, 
without using human-annotated labels~\citep{chen2020simple, hjelm2018learning, wu2018unsupervised, tian2019contrastive, chen2021exploring, gao2021simcse, bachman2019learning, 
oord2018representation, ye2019unsupervised, henaff2020data, misra2020self, caron2020unsupervised, haochen2021provable, caron2021emerging,li2021self, zbontar2021barlow, tsai2021note, bardes2021vicreg, tian2020makes, robinson2021contrastive}.
Such representations can be used for various downstream tasks with remarkable performance. 

\textbf{Theoretical understanding of self-supervised learning.\hspace{1ex}} The empirical success of contrastive learning has triggered a surge of theoretical explorations into the contrastive loss~\citep{arora2019theoretical, haochen2021provable, haochen2022beyond, tosh2020contrastive, tosh2021contrastive, lee2020predicting, wang2022chaos, nozawa2021understanding, huang2021towards, tian2022deep, hu2022your, tan2023contrastive}.
\citet{wang2020understanding} shed light on the optimal solutions of the InfoNCE loss, decomposing it as alignment term and uniformity term, contributing to a deeper comprehension of self-supervised learning. In \citet{haochen2021provable, shen2022connect, wang2022chaos, tan2023contrastive}, self-supervised learning methods are examined from a spectral graph perspective. \citet{zimmermann2021contrastive} provides a compelling probabilistic view of contrastive learning, suggesting that it can be seen as an inversion of the data-generating process, which assumes that the ground-truth marginal distribution of the latents of the generative process is uniform. 

Various theoretical studies have also investigated non-contrastive methods for self-supervised learning~\citep{wen2022mechanism, tian2021understanding, garrido2022duality, balestriero2022contrastive, tsai2021note, pokle2022contrasting, tao2022exploring, lee2021predicting}. \citet{garrido2022duality} establishes the duality between contrastive and non-contrastive methods. \citet{balestriero2022contrastive} reveal the connections between variants of SimCLR, Barlow Twins, and VICReg to ISOMAP, Canonical Correlation Analysis, and Laplacian Eigenmaps, respectively. 

\citet{tan2023information} also use matrix information theory to analyze non-contrastive methods, but they focus on applying $\alpha$-order mutual information to characterize the loss functions of Barlow Twins and spectral contrastive learning, and extend the analysis to MAE. By contrast, our paper focuses on incorporating alignment loss into the maximum entropy encoding framework.  

\textbf{Neural collapse and dimensional collapse. \hspace{1ex}} \citet{papyan2020prevalence} describe the intriguing phenomenon of Neural Collapse (NC), which manifests when training a classification network with cross-entropy loss. This phenomenon can be summarized that all the features of a single class converge to the mean of these features. In addition, the class-means form a simplex equiangular tight frame (ETF).  \citet{zhuo2023towards} advocate for a comprehensive theoretical understanding of non-contrastive learning through the mechanism of rank differential.

\section{Background}

Self-supervised learning (SSL) aims to learn meaningful representations from unlabeled data \( \{ x_i \}_{i=1}^{n} \), which can be used to enhance performance in various downstream tasks. Prominent SSL methods (architectures) like SimCLR, SimSiam, BYOL, Barlow Twins, and VICReg, employ 2-view augmentations: an online network \( \boldsymbol{f}_{\theta} \) and a target network \( \boldsymbol{f}_{\phi} \). Given a mini-batch \( \{ \mathbf{x}_i \}_{i=1}^{B} \), each data point \( \mathbf{x}_i \) is augmented using a random transformation \( \mathcal{T} \) from a predefined set \( \tau \) to obtain \( \mathbf{x}^{\prime}_i = \mathcal{T}(\mathbf{x}_i) \). These original and augmented data points are processed through the respective networks to generate feature representations \( \mathbf{z}_1^i \) and \( \mathbf{z}_2^i \), both residing in \( \mathbb{R}^d \). The resulting matrices \( \mathbf{Z}_1 \) and \( \mathbf{Z}_2 \in \R^{d \times B}\) form the basis for the training loss \( \mathcal{L}(\mathbf{Z}_1, \mathbf{Z}_2) \), which varies based on the learning paradigm—contrastive or non-contrastive.

\subsection{Contrastive Learning}

The idea of contrastive learning is to make the representation of similar objects align and dissimilar objects apart. One of the widely adopted losses in contrastive learning is the 
InfoNCE~\citep{oord2018representation} loss, 
where we use cosine similarity $\operatorname{sim}(\boldsymbol{u}, \boldsymbol{v})=\boldsymbol{u}^{\top} \boldsymbol{v} /(\|\boldsymbol{u}\|_2\|\boldsymbol{v}\|_2)$:
\begin{align*}
&\mathcal{L}_{\text{Contrastive}}(
\bfZ_1, \bfZ_2
) =\\&
\sum_{i}
-\log \frac{\exp \left(\operatorname{sim}\left(\bfZ_1^i, \bfZ_2^i\right) / \tau\right)}{\sum_{(p,k)\neq (1,i)} \exp \left(\operatorname{sim}\left(\bfZ_1^i, \bfZ_p^k\right) / \tau\right)
}.  
\end{align*}

\citet{wang2020understanding} showed that when the sample size $B$ goes to infinity, 
$\mathcal{L}_{\text{Contrastive}}$ can be decomposed into two parts. The first part 
is minimized if and only if $\bfZ$ is perfectly aligned (alignment loss), while if perfectly uniform encoders exist, they form the exact minimizers of the second part (uniformity loss).

\subsection{Non-contrastive Learning}
Given a matrix $\bfZ$, We define the total coding rate (TCR) \citep{cover1999elements, ma2007segmentation} loss  as:
\begin{equation}
\label{eq:tcr}
\mathcal{L}_{\text{TCR}}(\bfZ) = -\frac{1}{2} \log \operatorname{det} \left( \mathbf{I}_d + \frac{d}{B \epsilon^2} \mathbf{Z} \mathbf{Z}^{\top} \right).   
\end{equation}

Here $-(d+B) \mathcal{L}_{\text{TCR}}(\bfZ)$ 
captures the minimal number of bits for encoding $\bfZ$ up to $\epsilon$ distortion~ \citep{cover1999elements, ma2007segmentation}.

For the non-contrastive learning setting, we hope to maximize the total coding rate for the feature embeddings. Given that both the online and target networks are approximations of the feature map \(\boldsymbol{f}\), we can use the cross-covariance between \(\mathbf{Z}_1\) and \(\mathbf{Z}_2\) to approximate \(\mathbf{Z}\mathbf{Z}^\top\), resulting in the maximal entropy coding (MEC) loss~\citep{liu2022self}:
\begin{equation}
\label{eq:mec}
\begin{aligned}
 \mathcal{L}_{\text{MEC}} &= -\mu \log \operatorname{det}\left(\mathbf{I}_{d}+\frac{d}{B \epsilon^2} \mathbf{Z}_1 \mathbf{Z}_2^\top\right) \\&= -\mu \operatorname{tr}\left(\operatorname{log}\left(\mathbf{I}_{d}+\frac{d}{B \epsilon^2} \mathbf{Z}_1 \mathbf{Z}_2^\top\right)\right).  \\
\end{aligned}
\end{equation}

As discussed in~\citep{liu2022self},
MEC loss is a natural and general loss that subsumes many non-contrastive learning methods, including SimSiam~\citep{gao2021simcse}, BYOL~\citep{grill2020bootstrap}, Barlow Twins~\citep{zbontar2021barlow}, and VICReg~\citep{bardes2021vicreg}. 

\subsection{Matrix Information-Theoretic Quantities}
\label{sec:matrix-divergence}
Unlike Shannon entropy for random variables, the definition of matrix entropy is not necessarily unique. Specifically, within the domain of quantum information theory, matrix entropy is typically confined to positive semi-definite Hermitian matrices that possess a unit trace. However, our paper aims to extend this definition by incorporating positive semi-definite matrices that are not constrained by unit trace prerequisites, because the matrices may have various traces during optimization.

\begin{definition}[Matrix entropy for positive semi-definite matrices]
\label{def:matrix-entropy}
For a positive semi-definite matrix \( \mathbf{A} \in \R^{n \times n} \), the matrix entropy is defined as:
\begin{equation*}
\begin{aligned}
\operatorname{ME}(\mathbf{A}) &= -\operatorname{tr}(\mathbf{A} \log \mathbf{A}) + \operatorname{tr}(\mathbf{A}) \\&=-\sum_{i} \lambda_i \log \lambda_i + \sum_i \lambda_i.      
\end{aligned}
\end{equation*}
where \(\log\) denotes the principal matrix logarithm~\citep{higham2008functions}, and $\lambda_i$ denote the eigenvalues of matrix $\mathbf{A}$. For zero eigenvalues, we define $\log (0):= 0$. Our proposed matrix entropy generalizes the definition of von Neumann entropy~\citep{john1932mathematische,witten2020mini}, which is restricted to positive semi-definite matrices with unit trace. 
\end{definition}

\begin{definition}[Matrix KL divergence for positive semi-definite matrices~\citep{amari2014information}]
For two positive semi-definite matrices \( \mathbf{P}, \mathbf{Q} \in \R^{n \times n} \), the matrix KL divergence is defined as:
\begin{equation}
    \operatorname{MKL}(\mathbf{P} || \mathbf{Q}) = \operatorname{tr}(\mathbf{P} \log \mathbf{P} - \mathbf{P} \log \mathbf{Q} - \mathbf{P} + \mathbf{Q}).
\end{equation}
\end{definition}

This definition of matrix KL divergence generalizes the definition of quantum (von Neumann) KL divergence (relative entropy)~\citep{john1932mathematische,witten2020mini, bach2022information}.

Similar to classical cross-entropy based on Shannon information theory, we introduce the matrix cross-entropy as below:
\begin{definition}[Matrix Cross-Entropy (MCE) for positive semi-definite matrices]
For two positive semi-definite matrices \( \mathbf{P}, \mathbf{Q} \in \R^{n \times n} \), the matrix cross-entropy is defined as:
\begin{equation} 
\label{eq:mce}
\begin{aligned}
\mathrm{MCE}(\mathbf{P} , \mathbf{Q}) &= \operatorname{MKL}(\mathbf{P} || \mathbf{Q}) + \operatorname{ME}(\mathbf{P}) \\&= \operatorname{tr}(-\mathbf{P} \log \mathbf{Q} + \mathbf{Q}).
\end{aligned}
\end{equation}
\end{definition}

\begin{lemma}
\label{lem:psd-1}
For any non-zero matrix \( \mathbf{A} \in \mathbb{R}^{m \times n}\), \( \mathbf{A}\mathbf{A}^\top \) is positive semi-definite.
\end{lemma}

If not specified, we present proofs in the Appendix~\ref{sec:proofs}.
We employ matrix KL divergence and matrix cross-entropy (MCE) as canonical metrics for comparing positive semi-definite matrices since they have strong minimization properties, just like the classical KL divergence and cross-entropy in Shannon information theory (MKL and MCE are also asymmetric just like the classical ones).  

\begin{proposition}[Minimization property of matrix KL divergence]
\label{prop:optimal-point-matrix-kl}
For two positive semi-definite matrices $\mathbf{P}, \mathbf{Q} \in \mathbb{R}^{n \times n}$, the matrix $\mathbf{Q}$ that minimizes this divergence when $\mathbf{P}$ is fixed and \( \mathbf{Q} \) varies over all positive semi-definite matrices is \( \mathbf{P} \) itself, i.e.,
\begin{equation}
\operatorname{argmin}_{\mathbf{Q} \succ 0} \operatorname{MKL}(\mathbf{P} || \mathbf{Q}) = \mathbf{P}.    
\end{equation}
\end{proposition}

\begin{proposition}[Minimization property of matrix cross-entropy]
\label{prop:optimal-point-mce}
Let $\mathbf{P}, \mathbf{Q} \in \mathbb{R}^{n \times n}$ be positive semi-definite matrices. Then, the matrix $\mathbf{Q}$ that minimizes the matrix cross-entropy $\mathrm{MCE}(\mathbf{P}, \mathbf{Q})$ when $\mathbf{P}$ is fixed and $\mathbf{Q}$ varies over all positive semi-definite matrices is $\mathbf{P}$ itself, i.e.,
\begin{equation}
\operatorname{argmin}_{\mathbf{Q} \succ 0} \mathrm{MCE}(\mathbf{P}, \mathbf{Q}) = \mathbf{P}.
\end{equation}
\end{proposition}

\paragraph{Illustrative example. }

Consider a batch size \( B = 2 \) with two augmentation views. Let the representation matrices be \(\mathbf{Z}_1 = [\mathbf{a}_1, \mathbf{b}_1] \in \mathbb{R}^{2 \times 2}\) for the first view, and \(\mathbf{Z}_2 = [\mathbf{a}_2, \mathbf{b}_2] \in \mathbb{R}^{2 \times 2}\) for the second view. Suppose \(\mathbf{a}_1 = (1, 0)^{\top}\) and \(\mathbf{a}_2 = (0.8, 0.6)^{\top}\).

Consider two cases:

\begin{enumerate}
    \item \(\mathbf{b}_1 = (0, 1)^{\top}\) and \(\mathbf{b}_2 = (0.6, 0.8)^{\top}\).
    \item \(\mathbf{b}_1 = (0.6, 0.8)^{\top}\) and \(\mathbf{b}_2 = (0, 1)^{\top}\).
\end{enumerate}

In both cases, the typical alignment loss (e.g., BYOL-type MSE loss, \(\|\mathbf{a}_1 - \mathbf{a}_2\|^2 + \|\mathbf{b}_1 - \mathbf{b}_2\|^2\)) yields a value of 0.8. However, analyzing the covariance matrices \(\mathbf{Z}_1\mathbf{Z}_1^{\top}\) and \(\mathbf{Z}_2\mathbf{Z}_2^{\top}\) reveals more information:

\begin{itemize}
    \item For Case 1:
    \[
    \mathbf{Z}_1 \mathbf{Z}_1^{\top} = \begin{bmatrix} 1 & 0 \\ 0 & 1 \end{bmatrix}, \quad \mathbf{Z}_2 \mathbf{Z}_2^{\top} = \begin{bmatrix} 1 & 0.96 \\ 0.96 & 1 \end{bmatrix}
    \]
    Here $\operatorname{MKL}(\mathbf{Z}_1 \mathbf{Z}_1^{\top} \|\mathbf{Z}_2 \mathbf{Z}_2^{\top}) = 2.55$.
    \item For Case 2:
     $$
     \mathbf{Z}_1 \mathbf{Z}_1^{\top} = \begin{bmatrix} 1.36 & 0.48 \\ 0.48 & 0.64 \end{bmatrix}, \quad \mathbf{Z}_2 \mathbf{Z}_2^{\top} = \begin{bmatrix} 0.64 & 0.48 \\ 0.48 & 1.36 \end{bmatrix}
    $$
    Here $\operatorname{MKL}(\mathbf{Z}_1 \mathbf{Z}_1^{\top} \|\mathbf{Z}_2 \mathbf{Z}_2^{\top}) = 0.60$.
\end{itemize}

Matrix information theory, suitable for handling covariance and Gram matrices, allows us to capture these nuanced differences, enabling a more comprehensive understanding of the data representations. Aligning the matrices \(\mathbf{Z}_1 \mathbf{Z}_1^{\top}\) and \(\mathbf{Z}_2 \mathbf{Z}_2^{\top}\) is beneficial because it can reveal richer training signals beyond the typical vector alignment loss. Even when the vector alignment loss (e.g., BYOL-type MSE loss) yields the same value, the matrix alignment loss, measured by the matrix KL divergence \(\operatorname{MKL}(\mathbf{Z}_1 \mathbf{Z}_1^{\top} || \mathbf{Z}_2 \mathbf{Z}_2^{\top})\), can vary significantly between different cases. This variation provides additional insights into the structural alignment of the representations, ensuring that the learned features capture more detailed and discriminative information about the underlying data distribution. 

\subsection{Effective Rank}
\label{sec:effective-rank}

\citet{roy2007effective} introduced the concept of effective rank, which provides a real-valued extension of the classical rank.

\begin{definition}[Effective rank~\citep{roy2007effective}]
\label{def:erank}
The effective rank of a non-all-zero $\mathbf{A} \in \mathbb{C}^{n \times n}$, denoted $\operatorname{erank}(\mathbf{A})$, is defined as
\begin{equation}
\operatorname{erank}(\mathbf{A}) \triangleq \exp \left\{\operatorname{H}\left(p_1, p_2, \ldots, p_n\right)\right\},
\end{equation}

where $p_i = \frac{\sigma_i}{\sum_{k=1}^{n} \sigma_k}$, $\{\sigma_i | i = 1,\cdots,n \}$ are the singular values of $\mathbf{A}$, and $\text{H}$ is the Shannon entropy defined as $\operatorname{H}\left(p_1, p_2, \ldots, p_n\right) = -\sum_{i=1}^n p_i \log p_i$, with the convention that $0 \log 0 \triangleq 0$.
\end{definition}

\section{On TCR and Matrix KL Divergence} \label{sec:matrix-uniformity}

As we mentioned previously, there are two interesting questions about 
TCR. First, it is not immediately obvious why it is similar to the uniformity loss in contrastive learning. Secondly, one cannot easily integrate matrix alignment loss to directly align the feature covariance matrices in its formulation. In this section, we try to address both problems by building the connection between TCR and MCE/MKL. 

Given a batch of \( B \) data points \(\{ x_i \}^B_{i=1}\), and their $l_2$ normalized representations \( \mathbf{Z} = [\boldsymbol{f}(x_1), \cdots, \boldsymbol{f}(x_B)]\in \R^{d\times B} \). We design the following loss function to pursue uniformity, resulting in the following $\lambda$-regularized ($\lambda \geq 0$) Uniformity-MCE loss, which is well-defined due to Lemma~\ref{lem:psd-1}:
\begin{equation}
\operatorname{MCE}\left(\frac{1}{d}\mathbf{I}_d + \lambda \mathbf{I}_d, \frac{1}{B}\mathbf{Z} \mathbf{Z}^{\top} + \lambda \mathbf{I}_d\right),   
\end{equation}
This MCE-based uniformity loss definition captures the distance of regularized covariance matrix $\bfZ\bfZ^\top$ to the regularized (scaled) identity matrix and  We intentionally introduce the additional regularizer $\lambda\geq 0$ here, because we can prove the closed-form relationship between TCR and MCE/MKL 
for specific $\lambda > 0$, as follows. 

\begin{theorem}[Main Theorem]
\label{thm:mce-tcr-md}
Given a batch of \( B \) data points \(\{ x_i \}^B_{i=1}\), and their $l_2$ normalized representations \( \mathbf{Z} = [\boldsymbol{f}(x_1), \cdots, \boldsymbol{f}(x_B)]\in \R^{d\times B} \).
Assume that \( \lambda = \frac{\epsilon^2}{d} > 0 \) for $\epsilon, d$ in TCR loss~(\ref{eq:tcr}). Then,
\begin{align}
&\operatorname{MCE}\left(\frac{1}{d}\mathbf{I}_d + \lambda \mathbf{I}_d, \frac{1}{B}\mathbf{Z} \mathbf{Z}^{\top} + \lambda \mathbf{I}_d\right)\nonumber
\\=&(1 + d \lambda)\left(-\log \lambda + 1 + 2\mathcal{L}_{\text{TCR}}(\bfZ) \right),
\end{align}
\begin{align}
&\operatorname{MKL}\left(\left.\frac{1}{d}\mathbf{I}_d + \lambda \mathbf{I}_d \,\right | \left | \,\frac{1}{B}\mathbf{Z} \mathbf{Z}^{\top} + \lambda \mathbf{I}_d\right.\right)\nonumber
\\=& (1 + d \lambda)\left(\log \frac{1 + d \lambda}{\lambda d} + 2 \mathcal{L}_{\text{TCR}}(\bfZ)\right).
\end{align}
\end{theorem}
Theorem~\ref{thm:mce-tcr-md} shows a deep connection between TCR and MCE/MKL. Indeed, every TCR loss can be transformed into an MCE/MKL loss of the regularized covariance matrix to the scaled identity matrix (but not vice-versa since MCE/MKL has two operands while TCR has only one, and MCE/MKL can also be used for matrix alignment loss introduced in Section~\ref{sec:matrix-alignment}). 

\paragraph{Proof sketch, see the full proof in Appendix~\ref{proof:mce-tcr-md}.}

\begin{proof}
\label{proof:mce-tcr-md-sketch}
For notational simplicity, let
\begin{equation}
\mathcal{L}_{\text{UMCE}}\triangleq \operatorname{MCE}\left(\frac{1}{d}\mathbf{I}_d + \lambda \mathbf{I}_d, \frac{1}{B}\mathbf{Z} \mathbf{Z}^{\top} + \lambda \mathbf{I}_d\right)
\end{equation}

Using the definition of MCE, we get:
\[
\begin{aligned}
 &\operatorname{MCE}\left(\frac{1}{d} \mathbf{I}_d+\lambda \mathbf{I}_d, \frac{1}{B} \mathbf{Z} \mathbf{Z}^{\top}+\lambda \mathbf{I}_d\right)
 \\&=\operatorname{tr}\left(-\left(\frac{1}{d} \mathbf{I}_d+\lambda \mathbf{I}_d\right) \log \left(\frac{1}{B} \mathbf{Z} \mathbf{Z}^{\top}+\lambda \mathbf{I}_d\right)\right)\\
 &\qquad +\operatorname{tr}\left(\frac{1}{B} \mathbf{Z} \mathbf{Z}^{\top}+\lambda \mathbf{I}_d\right),  
\end{aligned}
\]
Now, let us divide and multiply by $\lambda$ of the term $-\log \left(\frac{1}{B} \mathbf{Z} \mathbf{Z}^{\top}+\lambda \mathbf{I}_d\right)$:
\[
-\log \left(\frac{1}{B} \mathbf{Z} \mathbf{Z}^{\top}+\frac{\epsilon^2}{d} \mathbf{I}_d\right)=-\log \left(\lambda\left(\frac{1}{\lambda B} \mathbf{Z} \mathbf{Z}^{\top}+ \mathbf{I}_d\right)\right),
\]

Upon substitution and simplification, we get:
\[
\begin{aligned}
\mathcal{L}_{\text{UMCE}} &= -(1 + d \lambda)\log \lambda +2 (1 + d \lambda) \mathcal{L}_{\text{TCR}}+1+d \lambda\\
&= (1 + d \lambda) \left(-\log \lambda + 1 + 2 \mathcal{L}_{\text{TCR}}\right).
\end{aligned}
\]
This matches the expression given in the proposition for $\mathcal{L}_{\text{UMCE}}$. The other part of the theorem on $\operatorname{MKL}\left(\left.\frac{1}{d}\mathbf{I}_d + \lambda \mathbf{I}_d \,\right | \left | \,\frac{1}{B}\mathbf{Z} \mathbf{Z}^{\top} + \lambda \mathbf{I}_d\right.\right)$ can be proved similarly. 
\end{proof}

From Proposition~\ref{prop:optimal-point-matrix-kl}, Proposition~\ref{prop:optimal-point-mce}, and Theorem~\ref{thm:mce-tcr-md}, we have the following theorem.

\begin{theorem}[Minimization property of TCR loss]
\label{thm:minimize-tcr}
Given a batch of \( B \) data points \(\{ x_i \}^B_{i=1}\), and their \( l_2 \)-normalized representations \( \mathbf{Z} = [\boldsymbol{f}(x_1), \cdots, \boldsymbol{f}(x_B)]\in \mathbb{R}^{d\times B} \), the global and unique minimizer under the constraint $\| \mathbf{z}_i\|^2_2 = 1$, for $i \in \{1,2,\cdots, B\}$ of TCR loss is $\frac{1}{B}\mathbf{Z} \mathbf{Z}^{\top} = \frac{1}{d}\mathbf{I}_d$.
\end{theorem}

In other words,  the covariance matrix that minimizes the TCR loss is the (scaled) identity matrix. 

\section{Matrix Uniformity and Alignment}
\label{sec:uniformity-and-alignment}

Based on the discussions in Section \ref{sec:matrix-uniformity}, we know that TCR loss can be replaced (up to constant terms and factors) by the MCE loss of the (regularized) covariance matrix to the scaled identity matrix. 
However, if we directly use the covariance matrix of $\bfZ$, the optimization process might be sub-optimal, as $\bfZ$ is not empirically aligned to have zero mean. Fortunately, the next theorem states that even if we center the covariance matrix, it will still be aligned with the scaled identity matrix at the maximal effective rank and unit trace. 

\begin{theorem}
\label{thm:covariance-matrix-effective-rank}
Let $\mathbf{x}$ be a random vector with a distribution supported on the unit hypersphere $S^{d-1}$. If the centered covariance matrix of $\mathbf{x}$, denoted by $\mathbf{C}(\mathbf{x})$, has the maximal possible effective rank $d$ and a trace of at least one, then the expected value of $\mathbf{x}$ will be zero, and $\mathbf{C}(\mathbf{x})$ will equal $\frac{1}{d}\mathbf{I}_d$.
\end{theorem}

To achieve matrix information-theoretic uniformity, we propose the following MCE-based uniformity loss, where $\mathbf{C}(\bfZ_1, \bfZ_2) = \frac{1}{B} \bfZ_1 \bfH_B \bfZ_2^{\top}$ (where \(\bfH_B = \bfI_B - \frac{1}{B} \mathbf{1_B 1_B}^{\top}\)) represents the centered sample covariance matrix for simplicity:
\begin{equation}
\label{eq:mkl-uniformity}
\begin{aligned}
\mathcal{L}_{\text{Matrix-Uniformity}}(\bfZ_1, \bfZ_2) &= \operatorname{MCE}\left(\frac{1}{d}\mathbf{I}_d, \bfC\left(\bfZ_1, \bfZ_2\right)\right).\\
\end{aligned}
\end{equation}

The next lemma states why $\bfH_B$ is the correct centering matrix to use. 
\begin{lemma}
\label{lem:covariance}
Let \(\bfZ_1, \bfZ_2 \in \R^{d \times B}\) where \(d\) is the dimensionality of the data and \(B\) is the number of samples. The cross-covariance matrix \(\bfC(\bfZ_1, \bfZ_2)\) can be expressed as:
\[
\bfC\left(\bfZ_1, \bfZ_2\right) = \frac{1}{B} \bfZ_1 \bfH_B \bfZ_2^{\top},
\]
where \(\bfH_B = \bfI_B - \frac{1}{B} \mathbf{1_B 1_B}^{\top}\) is the centering matrix.
\end{lemma}

For ease of optimization, a regularization term $\lambda \mathbf{I}_d$ may be added to this cross-covariance matrix to ensure it is non-singular. This adjustment aligns with TCR and MEC methods, differing mainly in mean normalization. An alternative approach is the auto-covariance uniformity loss $\sum_i \operatorname{MCE}\left(\frac{1}{d}\mathbf{I}_d, \bfC\left(\bfZ_i, \bfZ_i\right)\right)$, which is left for future exploration. 

\subsection{Matrix-SSL: Uniformity and Alignment}
\label{sec:matrix-alignment}

To directly pursue the alignment of representations in self-supervised learning, we propose the following loss function using the first-order alignment loss plus the matrix cross-entropy (MCE) between two covariance matrices: 
\begin{equation}
\label{eq:mkl-alignment}
\begin{aligned}
\mathcal{L}_{\text{Matrix-Alignment}}(\bfZ_1, \bfZ_2) &= -\operatorname{tr} \left(\bfC(\bfZ_1, \bfZ_2)\right)+ \\&\gamma \cdot \operatorname{MCE}\left(\mathbf{C}(\mathbf{Z}_1, \mathbf{Z}_1), \mathbf{C}(\mathbf{Z}_2, \mathbf{Z}_2)\right).\\
\end{aligned}
\end{equation}
\textbf{Discussion.} When the stop-gradient technique~\citep{gao2021simcse} is utilized on the target branch $\mathbf{Z}_1$, optimizing the MCE alignment loss is the same as optimizing the matrix KL divergence, since $\operatorname{MCE}(\mathbf{P}, \mathbf{Q}) = \operatorname{MKL}(\mathbf{P} || \mathbf{Q}) + \operatorname{ME}(\mathbf{P})$. We think this can partially answer the effectiveness of stop-gradient (details can be found in Appendix~\ref{sec:more-experiment-details}). 

As we have presented an improved loss for uniformity before, now generalizing \citet{wang2020understanding}'s understanding of contrastive learning, we propose the matrix information-theoretic uniformity and alignment framework to improve self-supervised learning:
\begin{equation}
\label{eq:matrix-ssl}
\begin{aligned}
\mathcal{L}_{\text{Matrix-SSL}} &= \mathcal{L}_{\text{Matrix-Uniformity}} + \mathcal{L}_{\text{Matrix-Alignment}}.\\
\end{aligned}
\end{equation}

\section{Effective Rank and Dimensional Collapse}
\label{sec:rank}
\citet{zhuo2023towards} find an intriguing phenomenon that during the optimization course of self-supervised learning, the effective rank of the (empirical) feature covariance matrix consistently increases.  This phenomenon can be analyzed with the following proposition. 

\begin{proposition}
\label{prop:matrix-entropy-kl-erank-relationship}
Matrix KL divergence of the covariance matrix to the uniform distribution $\frac{1}{d}\bfI_d$ has the following equality with connection to effective rank.
\begin{equation}
    \begin{aligned}
    \operatorname{erank}\left(\frac{1}{B}\mathbf{Z}\mathbf{Z}^{\top}\right) 
&=
\frac{d}{\exp{\left( 
\operatorname{MKL}\left(\frac{1}{B}\mathbf{Z}\mathbf{Z}^{\top} \,||\, \frac{1}{d}\mathbf{I}_d\right)\right) } }   
    \\&= \exp{( \operatorname{VNE}(\frac{1}{B}\mathbf{Z}\mathbf{Z}^{\top}))}.
    \end{aligned}
\end{equation}
\end{proposition}

Proposition~\ref{prop:matrix-entropy-kl-erank-relationship} captures the closed-form relationship among effective rank and matrix information-theoretic quantities. Note the batch auto-correlation matrix is a positive semi-definite matrix with all of its diagonal 1. As we have mentioned earlier, many dimension-contrastive losses can be understood from the matrix information-theoretic uniformity viewpoint. As such, during training the matrix KL divergence (MCE) minimizes, thus $\frac{1}{B}\mathbf{Z}\mathbf{Z}^{\top}$ is anticipated to progressively align more with $\frac{1}{d}\mathbf{I}_d$. By the fact that $\frac{1}{d}\mathbf{I}_d$ achieves the maximal possible (matrix) entropy, the rank-increasing phenomenon \citep{zhuo2023towards} can be well understood. Thus we may treat the effective rank as an exact metric to measure the extent of the dimensional collapse.

Feature representations acquired through a deep neural network employing a cross-entropy (CE) loss optimized by stochastic gradient descent, are capable of attaining zero loss~\citep{du2018gradient} with arbitrary label assignments~\citep{zhang2021understanding}. A phenomenon known as neural collapse (NC)~\citep{papyan2020prevalence} is observed when training of the neural network continues beyond zero loss with CE. Based on this, we propose to use effective rank as a unified tool to investigate the difference between supervised, contrastive, and non-contrastive methods, more details can be found in Appendix \ref{sec:measuring-dimensional-collapse}.

\section{Experiments}
\label{sec:implementation-linear-evaluation}

\subsection{Experimental Setup}

\textbf{Experiment details.} In this section, we implement our proposed Matrix-SSL method for self-supervised learning tasks on ImageNet~\citep{deng2009imagenet} dataset\footnote{The code is available at \url{https://github.com/yifanzhang-pro/Matrix-SSL}.}. We use precisely the same data augmentation protocols and hyperparameters as previous baselines such as BYOL~\citep{grill2020bootstrap}, SimSiam~\citep{chen2021exploring} and MEC~\citep{liu2022self}, etc. We augment each image \underline{twice} to get two different views during each training iteration. 
Similar to MEC~\citep{liu2022self}, we select one branch of the Siamese network as the online network and the other branch as the target network, updating the parameters using the exponential moving average method instead of loss backward. The pseudo-code for Matrix-SSL is shown as Algorithm~\ref{algo:pseudo-code}.

\textbf{Model architectures.} We use ResNet50~\citep{DBLP:journals/corr/HeZRS15} without the last linear layer as the backbone encoder, whose output feature dimension is 2048. Then we use a three-layer MLP with BN(Batch Normalization)~\citep{DBLP:journals/corr/IoffeS15} and ReLU~\citep{Nair2010RectifiedLU} as the projector after the encoder, and the projector maintains the feature dimension to be 2048 through three layers. For the online network, we apply an extra two-layer MLP with BN~\citep{DBLP:journals/corr/IoffeS15} and ReLU~\citep{Nair2010RectifiedLU} with hidden dimension 512 and output dimension 2048. 

\textbf{Data augmentations.} Our augmentation protocol consists of random cropping, color jittering, color dropping (grayscale), left-right flipping, Gaussian blurring, and polarization.

\textbf{Optimization and hyperparameters.} For pre-training, we use SGD optimizer with $2048$ batch size, $10^{-5}$ weight decay, $0.9$ momentum, and $4.0$ base learning rate, which is scheduled by cosine decay learning rate scheduler~\citep{DBLP:journals/corr/LoshchilovH16a}, to optimize the online network over training process. 
For the momentum used for the exponential moving average process, it is set to be $0.996$ to $1$ scheduled by another cosine scheduler. As for linear evaluation, we use LARS optimizer~\citep{DBLP:journals/corr/abs-1708-03888} with $4096$ batch size, $0.9$ momentum, no weight decay, and $0.03$ base learning rate scheduled by cosine decay learning rate scheduler, to train the linear layer over 100 epochs, and report the performance of last epoch. 

\begin{algorithm}[!htb]
\SetAlgoLined
    \PyComment{$f$: encoder network} \\
    \PyComment{$B$: batch size} \\
    \PyComment{$\mathcal{L}_{\text{Matrix-Uniformity}}$: Matrix-Uniformity loss} \\
    \PyComment{$\mathcal{L}_{\text{Matrix-Alignment}}$: Matrix-Alignment loss} \\
    \PyComment{$\gamma$: weight ratio used in alignment loss} \\
    \PyCode{for $X$ in loader:} \\
    \Indp   
        \PyComment{augment a batch of $B$ images in $X$} \\
        \PyCode{$X_1$, $X_2$ = aug(X), aug(X)} \\
        \PyCode{}\\
        \PyComment{calculate $l_2$ normalized embeddings} \\
        \PyCode{$\mathbf{Z}_1$, $\mathbf{Z}_2$ = $f(X_1)$, $f(X_2)$} \\
        \PyCode{}\\
        \PyComment{calculate uniformity and alignment loss} \\
        \PyCode{uniformity\_loss = $\mathcal{L}_{\text{Matrix-Uniformity}}(\mathbf{Z}_1, \mathbf{Z_2})$}\\
        \PyCode{alignment\_loss = $\mathcal{L}_{\text{Matrix-Alignment}(\gamma)}(\mathbf{Z}_1, \mathbf{Z_2})$}\\
        \PyCode{}\\ 
        \PyComment{calculate loss} \\
         \PyCode{loss = uniformity\_loss + alignment\_loss}\\
        \PyCode{}\\
        
        \PyComment{optimization step} \\
        \PyCode{loss.backward()}\\
        \PyCode{optimizer.step()}\\
        
    \Indm 
\caption{PyTorch-style Pseudo-code for Matrix-SSL\label{algo:pseudo-code}}
\end{algorithm}

\subsection{Evaluation Results}

\textbf{Linear evaluation.} We follow the standard linear evaluation protocol~\citep{chen2020simple, grill2020bootstrap, chen2021exploring}. We freeze the parameters of the backbone encoder and then connect a linear classification layer after it, and train the linear layer in the supervised setting. During training, each image is augmented by random cropping, resizing to 224$\times$224, and random horizontal flipping. At test time, each image is resized to 256$\times$256 and center cropped to 224$\times$ 224.

The Linear evaluation of the Top-1 accuracy result when pre-trained with 100, 200, and 400 epochs on ImageNet ~\citep{deng2009imagenet} dataset was shown in Table~\ref{tab:results-linear}. Notice that we use ResNet50 backbone as default for a fair comparison. Matrix-SSL consistently outperforms baselines across various pre-training epochs. 

\begin{table}[ht]
\centering
\vspace{-4.7mm}
\caption{\textbf{Linear evaluation} results (Top-1 accuracy) on ImageNet dataset with different pre-training epochs using ResNet50 backbone. \textbf{Bold} means the best, \underline{underline} means the second.}
\vspace{1mm}
\begin{tabular}{cccc}
\toprule
\multirow{2}{*}{Method} & \multicolumn{3}{c}{Pre-training Epochs} \\
\cmidrule{2-4}
 & 100 & 200 & 400
   \\ \midrule
SimCLR        & $66.5$ & $68.3$ & $69.8$\\
MoCo v2       & $67.4$ & $69.9$ & $71.0$\\
BYOL          & $66.5$ & $70.6$ & $73.2$\\
SwAV       & $66.5$ & $69.1$ & $70.7$\\
SimSiam       & $68.1$ & $70.0$ & $70.8$\\
Barlow Twins   & $67.3$ & $70.2$ & $71.8$\\
VICReg        & $68.6$ & $-$ & $-$ \\
MEC           & \underline{$70.6$} & \underline{$71.9$} & $\underline{73.5}$ \\
\midrule
Matrix-SSL {\tiny (Ours)} & $\mathbf{71.1}$ & $\mathbf{72.3}$ & $\mathbf{73.6}$\\ \bottomrule \\
\vspace{-7.7mm}
\end{tabular}
\label{tab:results-linear}
\end{table}

\textbf{Transfer learning.} Following the common protocol of previous works~\citep{chen2020improved, chen2021exploring, liu2022self}, we finetune the pre-trained models on MS-COCO~\citep{MSCOCO2014Lin} object detection and instance segmentation tasks. Table~\ref{tab:results-coco-detection} and Table~\ref{tab:results-coco-instance-segmentation} summarize experiment results of baseline models and Matrix-SSL. The experiment showed that Matrix-SSL consistently outperformed the baselines. It is worth mentioning that Matrix-SSL was only pre-trained for 400 epochs, but it already performed better than all the baselines pre-trained for 800 epochs. For a fair comparison, we employ a standard 2-view augmentation for all methods, more augmentation views such as $2 + 6$ views used in SwAV~\citep{caron2020unsupervised}, $2 + 2$ views used in I-VNE+~\citep{kim2023vne}, and $200$ views used in EMP-SSL~\citep{tong2023emp} would lead to superior performance and have been theoretically justified \citep{allen2020towards}. 

\begin{table}[!htbp]
\centering
\vspace{-2.7mm}
\caption{\textbf{Transfer learning on object detection tasks.} We finetune models pre-trained on ImageNet, with the same experiment settings as SimSiam and MEC for a fair comparison.}
\vspace{1mm}
\begin{tabular}{cccc}
\toprule
Method & $\text{AP}_{50}$ & AP & $\text{AP}_{75}$ \\
\midrule
SimCLR      & $57.7$ & $37.9$ & $40.9$ \\
MoCo v2     & $58.9$ & $39.3$ & $42.5$ \\
BYOL        & $57.8$ & $37.9$ & $40.9$ \\
SwAV        & $58.6$ & $38.4$& $41.3$ \\
Barlow Twins & $59.0$ & $39.2$ & $42.5$ \\
SimSiam     & $59.3$ & $39.2$ & $42.1$ \\
VICReg      & - & $40.0$ & - \\
MEC         & \underline{$59.8$} & \underline{$39.8$} & \underline{$43.2$} \\
\midrule
Matrix-SSL {\tiny (Ours)} & $\mathbf{60.8}$ & $\mathbf{41.0}$ & $\mathbf{44.2}$ \\
\bottomrule
\end{tabular}
\label{tab:results-coco-detection}
\end{table}

\begin{table}[!tb]
\centering
\vspace{-2.7mm}
\caption{\textbf{Transfer learning on instance segmentation tasks.} Employing a similar setup as in the detection tasks, we finetune models pre-trained on ImageNet. \textbf{Bold} means the best, \underline{underline} means the second.}
\vspace{1mm}
\begin{tabular}{cccc}
\toprule
Method & $\text{AP}^{\text{mask}}_{50}$ & $\text{AP}^{\text{mask}}$ & $\text{AP}^{\text{mask}}_{75}$ \\
\midrule
SimCLR      & $54.6$ & $33.3$ & $35.3$ \\
MoCo v2     & $55.8$ & $34.4$ & $36.5$ \\
BYOL        & $54.3$ & $33.2$ & $35.0$ \\
SwAV        & $55.2$ & $33.8$ & $35.9$ \\
Barlow Twins & $56.0$ & $34.3$ & $36.5$ \\
SimSiam     & $56.0$ & $34.4$ & $36.7$ \\
VICReg      & - & - & $36.7$ \\
MEC         & \underline{$56.3$} & \underline{$34.7$} & \underline{$36.8$} \\
\midrule
Matrix-SSL {\tiny(ours)} & $\mathbf{57.5}$ & $\mathbf{35.6}$ & $\mathbf{38.0}$ \\
\bottomrule
\end{tabular}
\label{tab:results-coco-instance-segmentation}
\end{table}

\textbf{Semi-supervised learning.} In semi-supervised learning tasks, we noticed that SwAV~\citep{caron2020unsupervised}, BarlowTwins~\citep{zbontar2021barlow}, and MEC~\citep{liu2022self} all chose different experiment settings and hyperparameters for this task. For a fair comparison, we directly used the same evaluation protocol as MEC and conducted a comparison with semi-supervised learning following 100-epoch pre-training against MEC, since MEC has the best performance in all the baselines on the semi-supervised task. From Table~\ref{tab:semi-supervised-learning}, we found that we achieved a significant improvement over MEC in 1\% semi-supervised learning, and we are comparable to MEC in the 10\% task.

\begin{table}[thb]
\centering
\vspace{-2.7mm}
\tiny
\caption{Results on semi-supervised learning tasks.}
\vspace{1mm}
\label{tab:semi-supervised-learning}
\resizebox{\columnwidth}{!}{
\begin{tabular}{lllll}
\toprule
Method & 1\% Acc@1 & 1\% Acc@5 & 10\% Acc@1 & 10\% Acc@5 \\
\midrule
MEC & 44.442 & 71.430 & 63.918 & $\mathbf{86.270}$ \\
Matrix-SSL & $\mathbf{45.158}$ & $\mathbf{71.848}$ & $\mathbf{63.940}$ & 86.172 \\
\bottomrule
\end{tabular}
}
\end{table}
\subsection{Ablation Studies}
\textbf{Alignment loss ratio.} We first investigate the impact of different alignment loss ratios (i.e., the $\gamma$ in Eqn.~\ref{eq:matrix-ssl}) on performance. We chose the 100-epoch pre-training task for the ablations, and the results are summarized in Table \ref{table:ablation-ratio}. Interestingly, setting $\gamma=1$ achieves the best linear evaluation performance, so we set the ratio to $1$ as the default.
\begin{table}[!htbp]
\vspace{-4.7mm}
\centering
\caption{Ablations on linear probing (\%) with various $\gamma$.}
\label{table:ablation-ratio}
\vspace{1mm}
\resizebox{0.9\columnwidth}{!}{
\begin{tabular}{c|ccccccc}
\toprule
\( \gamma \) & 0 & 0.3 & 0.5 & 0.6 & 1 & 1.3 & 1.5 \\
\midrule
Acc. & 70.6 & 70.7 & 71.0 & 70.9 & \textbf{71.1} & 70.8 & 70.8 \\
\bottomrule
\end{tabular}
}
\end{table}
\begin{table}[htb]
\centering
\vspace{-3.7mm}
\small
\caption{Results of different Taylor expansion orders for linear evaluation results.}
\vspace{1mm}
\label{tab:ablation-taylor}
\begin{tabular}{@{}cccc@{}}
\toprule
Taylor expansion order & 3 & 4 & 5
\\ \midrule
 Top-1 accuracy & $70.9$ & $\mathbf{71.1}$ & $\underline{71.0}$ \\
\midrule
\end{tabular}
\vspace{-4.3mm}
\end{table}

\textbf{Taylor expansion order.} We investigat the effect of the Taylor expansion order of matrix logarithm implementation (which is well-defined according to Theorem~\ref{thm:taylorhall2013}) on linear evaluation tasks, We keep most of the settings unchanged, except the Taylor expansion order. The results are summarized in Table \ref{tab:ablation-taylor}. As shown in the table, we found that Matrix-SSL performs best when the Taylor expansion order is $4$ in this setting, so we chose $4$ as the default parameter.

\section{Matrix Cross-Entropy for Large Language Models}
\label{sec:language}

We further introduce representation learning into the language modeling regime and use the matrix cross-entropy loss to fine-tune large language models by considering how to incorporate the information within feature representations in designing the loss functions.

The main intuition behind our method is that the similarity among the representation vector of different words (tokens) can be utilized to address the \textbf{synonym phenomenon} and \textbf{polysemous phenomenon} within natural language. For example, ``Let \textbf{'s} think step by step'' should be similar to ``Let \textbf{us} think step by step''. The classical cross-entropy loss hasn't captured this intricate part.

Consider the target distribution $\mathbf{p}$ given by the training corpus (which is typically one-hot) and the output distribution $\mathbf{q}$ given by the output of the language model.
Suppose we have $l_2$ normalized representation vectors $\mathbf{e_i} \in \R^{d}$ (column vectors) for tokens $v_i, i \in [n]$, where $n$ is the vocabulary size. One could use LM head embeddings, word embeddings, or any other representation vectors of the models. In our experiments, we use the LM head embeddings as default.

For auto-regressive LLMs with tokens $k \in \{1,2,\cdots,K\}$, we define positive semi-definite matrices $\mathbf{P} \in \R^{d \times d}$ and $\mathbf{Q} \in \R^{d \times d}$ as below:
$$
\mathbf{P}^{(k)} = \sum_i \left( p_i^{(k)} \cdot \mathbf{e}_i \mathbf{e}_i^{\top} \right), \quad \mathbf{Q}^{(k)} = \sum_j \left( q_j^{(k)} \cdot \mathbf{e}_j \mathbf{e}_j^{\top} \right).
$$
Then we define the following loss as our objective (since $\operatorname{tr}\left( \mathbf{Q}^{(k)}\right)$ are constant):
\begin{equation}
\begin{aligned}
\mathcal{L}_{\text{Matrix-LLM}} &= \sum_k \operatorname{CE}(\mathbf{p}^{(k)}, \mathbf{q}^{(k)}) + \sum_k\operatorname{MCE}(\mathbf{P}^{(k)}, \mathbf{Q}^{(k)}) \\
&=-\sum_k \sum_i p_i^{(k)} \log q_i^{(k)} -\sum_k\operatorname{tr} (\mathbf{P}^{(k)} \operatorname{log} \mathbf{Q}^{(k)}). \\
\end{aligned} 
\end{equation}

\subsection{Experiments on Fine-Tuning LLMs}

\begin{table}[htb]
\centering
\vspace{-3.7mm}
\caption{Performance comparison of various models on GSM8K~\citep{cobbe2021training} and MATH~\citep{hendrycks2021measuring} dataset. MM denotes instruction fine-tuned with the MetaMathQA dataset~\citep{yu2023metamath}.}
\vspace{1mm}
\resizebox{\columnwidth}{!}{
\begin{tabular}{l|c|c|l}
\toprule
\small
\textbf{Model} & \textbf{Meth.} & \textbf{GSM8K} (\%) & \textbf{MATH} (\%)\\
\hline
Minerva 8B & CE & 16.2 & 14.1 \\
Minerva 62B & CE & 52.4 & 27.6 \\
Minerva 540B & CE & 58.8 & \textbf{33.6} \\
WizardMath 7B & RL & 54.9 & 10.7 \\
WizardMath 13B & RL & 63.9 & 14.0 \\
WizardMath 70B & RL & \textbf{81.6} & 22.7 \\
LLaMA2 70B & CE & 56.8 & 13.5 \\
MetaMath 7B & CE & 66.5 & 19.8 \\
Llemma 7B & CE & 36.4 & 18.0 \\
Llemma-MM 7B & CE & \underline{69.2} & \underline{30.0} \\
\midrule
Llemma-MM 7B & $\mathcal{L}_{\text{Matrix-LLM}}$ & \textbf{72.3} (\textbf{+3.1}) & \textbf{30.2} (\textbf{+0.2})\\
\bottomrule
\end{tabular}
}
\label{tab:models_benchmarks}
\vspace{-4.3mm}
\end{table}

\textbf{Training Pipeline. } We use Llemma-7B~\citep{azerbayev2023llemma} as the base model, which is continued pre-trained on the Proof-Pile-2 dataset~\citep{paster2023openwebmath} using the CodeLLaMA model~\citep{touvron2023llama}. We then use $\mathcal{L}_{\text{Matrix-LLM}}$ to fine-tune it on the MetaMath dataset~\citep{yu2023metamath}.

\textbf{Experimental Results. } We evaluated the performance of different models on the mathematical reasoning dataset GSM8K~\citep{cobbe2021training} and MATH dataset~\citep{hendrycks2021measuring}, using different loss functions and training methods. The results are shown in Table~\ref{tab:models_benchmarks}. We compared our results against baseline methods, including Minerva \citep{lewkowycz2022solving}, WizardMath~\citep{luo2023wizardmath}, and Llemma~\citep{azerbayev2023llemma} fine-tuned with MetaMath~\citep{yu2023metamath} dataset using classical cross-entropy (CE).

\section{Conclusion}

In this paper, we provide a matrix information-theoretic perspective for understanding and improving self-supervised learning methods. We are confident that our perspective will not only offer a refined and alternative comprehension of self-supervised learning methods but will also act as a catalyst for the design of increasingly robust and effective algorithms in the future.

\section*{Acknowledgment}
Yang Yuan is supported by the Ministry of Science and Technology of the People's Republic of China, the 2030 Innovation Megaprojects ``Program on New Generation Artificial Intelligence'' (Grant No. 2021AAA0150000). 
Weiran Huang is supported by the 2023 CCF-Baidu Open Fund and Microsoft Research Asia. 

We would also like to express our sincere gratitude to the reviewers of ICML 2024 for their insightful and constructive feedback. Their valuable comments have greatly contributed to improving the quality of our work.

\section*{Impact Statement}

This paper aims to contribute to the advancement of machine learning by introducing novel approaches to self-supervised learning. While this work primarily seeks to enrich research within the field, it acknowledges the potential broader societal implications inherent in any advancement in machine learning. However, specific societal consequences are not directly foreseeable at this stage.


\bibliography{reference}

\begin{thebibliography}{88}
\providecommand{\natexlab}[1]{#1}
\providecommand{\url}[1]{\texttt{#1}}
\expandafter\ifx\csname urlstyle\endcsname\relax
  \providecommand{\doi}[1]{doi: #1}\else
  \providecommand{\doi}{doi: \begingroup \urlstyle{rm}\Url}\fi

\bibitem[Allen-Zhu \& Li(2020)Allen-Zhu and Li]{allen2020towards}
Allen-Zhu, Z. and Li, Y.
\newblock Towards understanding ensemble, knowledge distillation and
  self-distillation in deep learning.
\newblock \emph{arXiv preprint arXiv:2012.09816}, 2020.

\bibitem[Amari(2014)]{amari2014information}
Amari, S.-i.
\newblock Information geometry of positive measures and positive-definite
  matrices: Decomposable dually flat structure.
\newblock \emph{Entropy}, 16\penalty0 (4):\penalty0 2131--2145, 2014.

\bibitem[Arora et~al.(2019)Arora, Khandeparkar, Khodak, Plevrakis, and
  Saunshi]{arora2019theoretical}
Arora, S., Khandeparkar, H., Khodak, M., Plevrakis, O., and Saunshi, N.
\newblock A theoretical analysis of contrastive unsupervised representation
  learning.
\newblock In \emph{International Conference on Machine Learning}, 2019.

\bibitem[Azerbayev et~al.(2023)Azerbayev, Schoelkopf, Paster, Santos, McAleer,
  Jiang, Deng, Biderman, and Welleck]{azerbayev2023llemma}
Azerbayev, Z., Schoelkopf, H., Paster, K., Santos, M.~D., McAleer, S., Jiang,
  A.~Q., Deng, J., Biderman, S., and Welleck, S.
\newblock Llemma: An open language model for mathematics.
\newblock \emph{arXiv preprint arXiv:2310.10631}, 2023.

\bibitem[Bach(2022)]{bach2022information}
Bach, F.
\newblock Information theory with kernel methods.
\newblock \emph{IEEE Transactions on Information Theory}, 2022.

\bibitem[Bachman et~al.(2019)Bachman, Hjelm, and
  Buchwalter]{bachman2019learning}
Bachman, P., Hjelm, R.~D., and Buchwalter, W.
\newblock Learning representations by maximizing mutual information across
  views.
\newblock \emph{arXiv preprint arXiv:1906.00910}, 2019.

\bibitem[Balestriero \& LeCun(2022)Balestriero and
  LeCun]{balestriero2022contrastive}
Balestriero, R. and LeCun, Y.
\newblock Contrastive and non-contrastive self-supervised learning recover
  global and local spectral embedding methods.
\newblock \emph{arXiv preprint arXiv:2205.11508}, 2022.

\bibitem[Bardes et~al.(2021)Bardes, Ponce, and LeCun]{bardes2021vicreg}
Bardes, A., Ponce, J., and LeCun, Y.
\newblock Vicreg: Variance-invariance-covariance regularization for
  self-supervised learning.
\newblock \emph{arXiv preprint arXiv:2105.04906}, 2021.

\bibitem[Caron et~al.(2020)Caron, Misra, Mairal, Goyal, Bojanowski, and
  Joulin]{caron2020unsupervised}
Caron, M., Misra, I., Mairal, J., Goyal, P., Bojanowski, P., and Joulin, A.
\newblock Unsupervised learning of visual features by contrasting cluster
  assignments.
\newblock \emph{Advances in Neural Information Processing Systems},
  33:\penalty0 9912--9924, 2020.

\bibitem[Caron et~al.(2021)Caron, Touvron, Misra, J{\'e}gou, Mairal,
  Bojanowski, and Joulin]{caron2021emerging}
Caron, M., Touvron, H., Misra, I., J{\'e}gou, H., Mairal, J., Bojanowski, P.,
  and Joulin, A.
\newblock Emerging properties in self-supervised vision transformers.
\newblock In \emph{Proceedings of the IEEE/CVF international conference on
  computer vision}, pp.\  9650--9660, 2021.

\bibitem[Chen et~al.(2020{\natexlab{a}})Chen, Kornblith, Norouzi, and
  Hinton]{chen2020simple}
Chen, T., Kornblith, S., Norouzi, M., and Hinton, G.
\newblock A simple framework for contrastive learning of visual
  representations.
\newblock In \emph{International Conference on Machine Learning}, pp.\
  1597--1607. PMLR, 2020{\natexlab{a}}.

\bibitem[Chen \& He(2021)Chen and He]{chen2021exploring}
Chen, X. and He, K.
\newblock Exploring simple siamese representation learning.
\newblock In \emph{Proceedings of the IEEE/CVF conference on Computer Vision
  and Pattern Recognition}, pp.\  15750--15758, 2021.

\bibitem[Chen et~al.(2020{\natexlab{b}})Chen, Fan, Girshick, and
  He]{chen2020improved}
Chen, X., Fan, H., Girshick, R., and He, K.
\newblock Improved baselines with momentum contrastive learning.
\newblock \emph{arXiv preprint arXiv:2003.04297}, 2020{\natexlab{b}}.

\bibitem[Cobbe et~al.(2021)Cobbe, Kosaraju, Bavarian, Chen, Jun, Kaiser,
  Plappert, Tworek, Hilton, Nakano, et~al.]{cobbe2021training}
Cobbe, K., Kosaraju, V., Bavarian, M., Chen, M., Jun, H., Kaiser, L., Plappert,
  M., Tworek, J., Hilton, J., Nakano, R., et~al.
\newblock Training verifiers to solve math word problems.
\newblock \emph{arXiv preprint arXiv:2110.14168}, 2021.

\bibitem[Cover(1999)]{cover1999elements}
Cover, T.~M.
\newblock \emph{Elements of information theory}.
\newblock John Wiley \& Sons, 1999.

\bibitem[Deng et~al.(2009)Deng, Dong, Socher, Li, Li, and
  Fei-Fei]{deng2009imagenet}
Deng, J., Dong, W., Socher, R., Li, L.-J., Li, K., and Fei-Fei, L.
\newblock Imagenet: A large-scale hierarchical image database.
\newblock In \emph{2009 IEEE conference on computer vision and pattern
  recognition}, pp.\  248--255. Ieee, 2009.

\bibitem[Du et~al.(2018)Du, Zhai, Poczos, and Singh]{du2018gradient}
Du, S.~S., Zhai, X., Poczos, B., and Singh, A.
\newblock Gradient descent provably optimizes over-parameterized neural
  networks.
\newblock \emph{arXiv preprint arXiv:1810.02054}, 2018.

\bibitem[Galanti et~al.(2021)Galanti, Gy{\"o}rgy, and Hutter]{galanti2021role}
Galanti, T., Gy{\"o}rgy, A., and Hutter, M.
\newblock On the role of neural collapse in transfer learning.
\newblock \emph{arXiv preprint arXiv:2112.15121}, 2021.

\bibitem[Gao et~al.(2021)Gao, Yao, and Chen]{gao2021simcse}
Gao, T., Yao, X., and Chen, D.
\newblock Simcse: Simple contrastive learning of sentence embeddings.
\newblock \emph{arXiv preprint arXiv:2104.08821}, 2021.

\bibitem[Garrido et~al.(2022)Garrido, Chen, Bardes, Najman, and
  Lecun]{garrido2022duality}
Garrido, Q., Chen, Y., Bardes, A., Najman, L., and Lecun, Y.
\newblock On the duality between contrastive and non-contrastive
  self-supervised learning.
\newblock \emph{arXiv preprint arXiv:2206.02574}, 2022.

\bibitem[Grill et~al.(2020)Grill, Strub, Altch{\'e}, Tallec, Richemond,
  Buchatskaya, Doersch, Avila~Pires, Guo, Gheshlaghi~Azar,
  et~al.]{grill2020bootstrap}
Grill, J.-B., Strub, F., Altch{\'e}, F., Tallec, C., Richemond, P.,
  Buchatskaya, E., Doersch, C., Avila~Pires, B., Guo, Z., Gheshlaghi~Azar, M.,
  et~al.
\newblock Bootstrap your own latent-a new approach to self-supervised learning.
\newblock \emph{Advances in Neural Iformation Processing Systems}, 33:\penalty0
  21271--21284, 2020.

\bibitem[Hall(2013)]{hall2013lie}
Hall, B.~C.
\newblock \emph{Lie groups, Lie algebras, and representations}.
\newblock Springer, 2013.

\bibitem[HaoChen et~al.(2021)HaoChen, Wei, Gaidon, and Ma]{haochen2021provable}
HaoChen, J.~Z., Wei, C., Gaidon, A., and Ma, T.
\newblock Provable guarantees for self-supervised deep learning with spectral
  contrastive loss.
\newblock \emph{Advances in Neural Information Processing Systems},
  34:\penalty0 5000--5011, 2021.

\bibitem[HaoChen et~al.(2022)HaoChen, Wei, Kumar, and Ma]{haochen2022beyond}
HaoChen, J.~Z., Wei, C., Kumar, A., and Ma, T.
\newblock Beyond separability: Analyzing the linear transferability of
  contrastive representations to related subpopulations.
\newblock \emph{Advances in Neural Information Processing Systems}, 2022.

\bibitem[He et~al.(2015)He, Zhang, Ren, and Sun]{DBLP:journals/corr/HeZRS15}
He, K., Zhang, X., Ren, S., and Sun, J.
\newblock Deep residual learning for image recognition.
\newblock \emph{CoRR}, abs/1512.03385, 2015.
\newblock URL \url{http://arxiv.org/abs/1512.03385}.

\bibitem[He et~al.(2020)He, Fan, Wu, Xie, and Girshick]{he2020momentum}
He, K., Fan, H., Wu, Y., Xie, S., and Girshick, R.
\newblock Momentum contrast for unsupervised visual representation learning.
\newblock In \emph{Proceedings of the IEEE/CVF conference on computer vision
  and pattern recognition}, pp.\  9729--9738, 2020.

\bibitem[Henaff(2020)]{henaff2020data}
Henaff, O.
\newblock Data-efficient image recognition with contrastive predictive coding.
\newblock In \emph{International Conference on Machine Learning}, pp.\
  4182--4192. PMLR, 2020.

\bibitem[Hendrycks et~al.(2021)Hendrycks, Burns, Kadavath, Arora, Basart, Tang,
  Song, and Steinhardt]{hendrycks2021measuring}
Hendrycks, D., Burns, C., Kadavath, S., Arora, A., Basart, S., Tang, E., Song,
  D., and Steinhardt, J.
\newblock Measuring mathematical problem solving with the math dataset.
\newblock \emph{arXiv preprint arXiv:2103.03874}, 2021.

\bibitem[Higham(2008)]{higham2008functions}
Higham, N.~J.
\newblock \emph{Functions of matrices: theory and computation}.
\newblock SIAM, 2008.

\bibitem[Hjelm et~al.(2018)Hjelm, Fedorov, Lavoie-Marchildon, Grewal, Bachman,
  Trischler, and Bengio]{hjelm2018learning}
Hjelm, R.~D., Fedorov, A., Lavoie-Marchildon, S., Grewal, K., Bachman, P.,
  Trischler, A., and Bengio, Y.
\newblock Learning deep representations by mutual information estimation and
  maximization.
\newblock In \emph{International Conference on Learning Representations}, 2018.

\bibitem[Hu et~al.(2022)Hu, Liu, Zhou, Wang, and Huang]{hu2022your}
Hu, T., Liu, Z., Zhou, F., Wang, W., and Huang, W.
\newblock Your contrastive learning is secretly doing stochastic neighbor
  embedding.
\newblock \emph{arXiv preprint arXiv:2205.14814}, 2022.

\bibitem[Hua(2021)]{Hua2021SimSiam}
Hua, T.
\newblock Simsiam.
\newblock \url{https://github.com/PatrickHua/SimSiam}, 2021.

\bibitem[Huang et~al.(2021)Huang, Yi, and Zhao]{huang2021towards}
Huang, W., Yi, M., and Zhao, X.
\newblock Towards the generalization of contrastive self-supervised learning.
\newblock \emph{arXiv preprint arXiv:2111.00743}, 2021.

\bibitem[Ioffe \& Szegedy(2015)Ioffe and Szegedy]{DBLP:journals/corr/IoffeS15}
Ioffe, S. and Szegedy, C.
\newblock Batch normalization: Accelerating deep network training by reducing
  internal covariate shift.
\newblock \emph{CoRR}, abs/1502.03167, 2015.
\newblock URL \url{http://arxiv.org/abs/1502.03167}.

\bibitem[Kim et~al.(2023)Kim, Kang, Hwang, Shin, and Rhee]{kim2023vne}
Kim, J., Kang, S., Hwang, D., Shin, J., and Rhee, W.
\newblock Vne: An effective method for improving deep representation by
  manipulating eigenvalue distribution.
\newblock In \emph{Proceedings of the IEEE/CVF Conference on Computer Vision
  and Pattern Recognition}, pp.\  3799--3810, 2023.

\bibitem[Krizhevsky et~al.(2009)Krizhevsky, Hinton,
  et~al.]{krizhevsky2009learning}
Krizhevsky, A., Hinton, G., et~al.
\newblock Learning multiple layers of features from tiny images.
\newblock \emph{Citeseer}, 2009.

\bibitem[Lee et~al.(2020)Lee, Lei, Saunshi, and Zhuo]{lee2020predicting}
Lee, J.~D., Lei, Q., Saunshi, N., and Zhuo, J.
\newblock Predicting what you already know helps: Provable self-supervised
  learning.
\newblock \emph{arXiv preprint arXiv:2008.01064}, 2020.

\bibitem[Lee et~al.(2021)Lee, Lei, Saunshi, and Zhuo]{lee2021predicting}
Lee, J.~D., Lei, Q., Saunshi, N., and Zhuo, J.
\newblock Predicting what you already know helps: Provable self-supervised
  learning.
\newblock \emph{Advances in Neural Information Processing Systems},
  34:\penalty0 309--323, 2021.

\bibitem[Lewkowycz et~al.(2022)Lewkowycz, Andreassen, Dohan, Dyer, Michalewski,
  Ramasesh, Slone, Anil, Schlag, Gutman-Solo, Wu, Neyshabur, Gur-Ari, and
  Misra]{lewkowycz2022solving}
Lewkowycz, A., Andreassen, A., Dohan, D., Dyer, E., Michalewski, H., Ramasesh,
  V., Slone, A., Anil, C., Schlag, I., Gutman-Solo, T., Wu, Y., Neyshabur, B.,
  Gur-Ari, G., and Misra, V.
\newblock Solving quantitative reasoning problems with language models, 2022.

\bibitem[Li et~al.(2021)Li, Pogodin, Sutherland, and Gretton]{li2021self}
Li, Y., Pogodin, R., Sutherland, D.~J., and Gretton, A.
\newblock Self-supervised learning with kernel dependence maximization.
\newblock \emph{Advances in Neural Information Processing Systems},
  34:\penalty0 15543--15556, 2021.

\bibitem[Lin et~al.(2014)Lin, Maire, Belongie, Bourdev, Girshick, Hays, Perona,
  Ramanan, Doll{\'{a}}r, and Zitnick]{MSCOCO2014Lin}
Lin, T., Maire, M., Belongie, S.~J., Bourdev, L.~D., Girshick, R.~B., Hays, J.,
  Perona, P., Ramanan, D., Doll{\'{a}}r, P., and Zitnick, C.~L.
\newblock Microsoft {COCO:} common objects in context.
\newblock \emph{CoRR}, abs/1405.0312, 2014.
\newblock URL \url{http://arxiv.org/abs/1405.0312}.

\bibitem[Liu et~al.(2022)Liu, Wang, Li, and Wang]{liu2022self}
Liu, X., Wang, Z., Li, Y.-L., and Wang, S.
\newblock Self-supervised learning via maximum entropy coding.
\newblock \emph{Advances in Neural Information Processing Systems},
  35:\penalty0 34091--34105, 2022.

\bibitem[Logeswaran \& Lee(2018)Logeswaran and Lee]{logeswaran2018efficient}
Logeswaran, L. and Lee, H.
\newblock An efficient framework for learning sentence representations.
\newblock \emph{arXiv preprint arXiv:1803.02893}, 2018.

\bibitem[Loshchilov \& Hutter(2016)Loshchilov and
  Hutter]{DBLP:journals/corr/LoshchilovH16a}
Loshchilov, I. and Hutter, F.
\newblock {SGDR:} stochastic gradient descent with restarts.
\newblock \emph{CoRR}, abs/1608.03983, 2016.
\newblock URL \url{http://arxiv.org/abs/1608.03983}.

\bibitem[Luo et~al.(2023)Luo, Sun, Xu, Zhao, Lou, Tao, Geng, Lin, Chen, and
  Zhang]{luo2023wizardmath}
Luo, H., Sun, Q., Xu, C., Zhao, P., Lou, J., Tao, C., Geng, X., Lin, Q., Chen,
  S., and Zhang, D.
\newblock Wizardmath: Empowering mathematical reasoning for large language
  models via reinforced evol-instruct, 2023.

\bibitem[Ma et~al.(2023)Ma, You, Reddi, Jayasumana, Jain, Yu, Chang, and
  Kumar]{ma2023do}
Ma, J., You, C., Reddi, S.~J., Jayasumana, S., Jain, H., Yu, F., Chang, S.-F.,
  and Kumar, S.
\newblock Do we need neural collapse? learning diverse features for
  fine-grained and long-tail classification.
\newblock \emph{OpenReviewNet}, 2023.

\bibitem[Ma et~al.(2007)Ma, Derksen, Hong, and Wright]{ma2007segmentation}
Ma, Y., Derksen, H., Hong, W., and Wright, J.
\newblock Segmentation of multivariate mixed data via lossy data coding and
  compression.
\newblock \emph{IEEE transactions on pattern analysis and machine
  intelligence}, 29\penalty0 (9):\penalty0 1546--1562, 2007.

\bibitem[Misra \& Maaten(2020)Misra and Maaten]{misra2020self}
Misra, I. and Maaten, L. v.~d.
\newblock Self-supervised learning of pretext-invariant representations.
\newblock In \emph{Proceedings of the IEEE/CVF Conference on Computer Vision
  and Pattern Recognition}, pp.\  6707--6717, 2020.

\bibitem[Nair \& Hinton(2010)Nair and Hinton]{Nair2010RectifiedLU}
Nair, V. and Hinton, G.~E.
\newblock Rectified linear units improve restricted boltzmann machines.
\newblock In \emph{International Conference on Machine Learning}, 2010.

\bibitem[Nozawa \& Sato(2021)Nozawa and Sato]{nozawa2021understanding}
Nozawa, K. and Sato, I.
\newblock Understanding negative samples in instance discriminative
  self-supervised representation learning.
\newblock \emph{Advances in Neural Information Processing Systems},
  34:\penalty0 5784--5797, 2021.

\bibitem[Oord et~al.(2018)Oord, Li, and Vinyals]{oord2018representation}
Oord, A. v.~d., Li, Y., and Vinyals, O.
\newblock Representation learning with contrastive predictive coding.
\newblock \emph{arXiv preprint arXiv:1807.03748}, 2018.

\bibitem[Papyan et~al.(2020)Papyan, Han, and Donoho]{papyan2020prevalence}
Papyan, V., Han, X., and Donoho, D.~L.
\newblock Prevalence of neural collapse during the terminal phase of deep
  learning training.
\newblock \emph{Proceedings of the National Academy of Sciences}, 117\penalty0
  (40):\penalty0 24652--24663, 2020.

\bibitem[Paster et~al.(2023)Paster, Santos, Azerbayev, and
  Ba]{paster2023openwebmath}
Paster, K., Santos, M.~D., Azerbayev, Z., and Ba, J.
\newblock Openwebmath: An open dataset of high-quality mathematical web text.
\newblock \emph{arXiv preprint arXiv:2310.06786}, 2023.

\bibitem[Pedregosa et~al.(2011)Pedregosa, Varoquaux, Gramfort, Michel, Thirion,
  Grisel, Blondel, Prettenhofer, Weiss, Dubourg, et~al.]{pedregosa2011scikit}
Pedregosa, F., Varoquaux, G., Gramfort, A., Michel, V., Thirion, B., Grisel,
  O., Blondel, M., Prettenhofer, P., Weiss, R., Dubourg, V., et~al.
\newblock Scikit-learn: Machine learning in python.
\newblock \emph{Journal of machine learning research}, 12\penalty0
  (Oct):\penalty0 2825--2830, 2011.

\bibitem[Pokle et~al.(2022)Pokle, Tian, Li, and Risteski]{pokle2022contrasting}
Pokle, A., Tian, J., Li, Y., and Risteski, A.
\newblock Contrasting the landscape of contrastive and non-contrastive
  learning.
\newblock \emph{arXiv preprint arXiv:2203.15702}, 2022.

\bibitem[Robinson et~al.(2021)Robinson, Chuang, Sra, and
  Jegelka]{robinson2021contrastive}
Robinson, J.~D., Chuang, C.-Y., Sra, S., and Jegelka, S.
\newblock Contrastive learning with hard negative samples.
\newblock In \emph{ICLR}, 2021.

\bibitem[Roy \& Vetterli(2007)Roy and Vetterli]{roy2007effective}
Roy, O. and Vetterli, M.
\newblock The effective rank: A measure of effective dimensionality.
\newblock In \emph{2007 15th European signal processing conference}, pp.\
  606--610. IEEE, 2007.

\bibitem[Shen et~al.(2022)Shen, Jones, Kumar, Xie, HaoChen, Ma, and
  Liang]{shen2022connect}
Shen, K., Jones, R.~M., Kumar, A., Xie, S.~M., HaoChen, J.~Z., Ma, T., and
  Liang, P.
\newblock Connect, not collapse: Explaining contrastive learning for
  unsupervised domain adaptation.
\newblock In \emph{International Conference on Machine Learning}, pp.\
  19847--19878. PMLR, 2022.

\bibitem[Tan et~al.(2023{\natexlab{a}})Tan, Yang, Huang, Yuan, and
  Zhang]{tan2023information}
Tan, Z., Yang, J., Huang, W., Yuan, Y., and Zhang, Y.
\newblock Information flow in self-supervised learning.
\newblock \emph{arXiv preprint arXiv:2309.17281}, 2023{\natexlab{a}}.

\bibitem[Tan et~al.(2023{\natexlab{b}})Tan, Zhang, Yang, and
  Yuan]{tan2023contrastive}
Tan, Z., Zhang, Y., Yang, J., and Yuan, Y.
\newblock Contrastive learning is spectral clustering on similarity graph.
\newblock \emph{arXiv preprint arXiv:2303.15103}, 2023{\natexlab{b}}.

\bibitem[Tao et~al.(2022)Tao, Wang, Zhu, Dong, Song, Huang, and
  Dai]{tao2022exploring}
Tao, C., Wang, H., Zhu, X., Dong, J., Song, S., Huang, G., and Dai, J.
\newblock Exploring the equivalence of siamese self-supervised learning via a
  unified gradient framework.
\newblock In \emph{Proceedings of the IEEE/CVF Conference on Computer Vision
  and Pattern Recognition}, pp.\  14431--14440, 2022.

\bibitem[Tian(2022)]{tian2022deep}
Tian, Y.
\newblock Deep contrastive learning is provably (almost) principal component
  analysis.
\newblock \emph{arXiv preprint arXiv:2201.12680}, 2022.

\bibitem[Tian et~al.(2019)Tian, Krishnan, and Isola]{tian2019contrastive}
Tian, Y., Krishnan, D., and Isola, P.
\newblock Contrastive multiview coding.
\newblock \emph{arXiv preprint arXiv:1906.05849}, 2019.

\bibitem[Tian et~al.(2020{\natexlab{a}})Tian, Krishnan, and
  Isola]{tian2020contrastive}
Tian, Y., Krishnan, D., and Isola, P.
\newblock Contrastive multiview coding.
\newblock In \emph{Computer Vision--ECCV 2020: 16th European Conference,
  Glasgow, UK, August 23--28, 2020, Proceedings, Part XI 16}, pp.\  776--794.
  Springer, 2020{\natexlab{a}}.

\bibitem[Tian et~al.(2020{\natexlab{b}})Tian, Sun, Poole, Krishnan, Schmid, and
  Isola]{tian2020makes}
Tian, Y., Sun, C., Poole, B., Krishnan, D., Schmid, C., and Isola, P.
\newblock What makes for good views for contrastive learning.
\newblock \emph{arXiv preprint arXiv:2005.10243}, 2020{\natexlab{b}}.

\bibitem[Tian et~al.(2021)Tian, Chen, and Ganguli]{tian2021understanding}
Tian, Y., Chen, X., and Ganguli, S.
\newblock Understanding self-supervised learning dynamics without contrastive
  pairs.
\newblock In \emph{International Conference on Machine Learning}, pp.\
  10268--10278. PMLR, 2021.

\bibitem[Tong et~al.(2023)Tong, Chen, Ma, and Lecun]{tong2023emp}
Tong, S., Chen, Y., Ma, Y., and Lecun, Y.
\newblock Emp-ssl: Towards self-supervised learning in one training epoch.
\newblock \emph{arXiv preprint arXiv:2304.03977}, 2023.

\bibitem[Tosh et~al.(2020)Tosh, Krishnamurthy, and Hsu]{tosh2020contrastive}
Tosh, C., Krishnamurthy, A., and Hsu, D.
\newblock Contrastive estimation reveals topic posterior information to linear
  models.
\newblock \emph{arXiv:2003.02234}, 2020.

\bibitem[Tosh et~al.(2021)Tosh, Krishnamurthy, and Hsu]{tosh2021contrastive}
Tosh, C., Krishnamurthy, A., and Hsu, D.
\newblock Contrastive learning, multi-view redundancy, and linear models.
\newblock In \emph{Algorithmic Learning Theory}, pp.\  1179--1206. PMLR, 2021.

\bibitem[Touvron et~al.(2023)Touvron, Martin, Stone, Albert, Almahairi, Babaei,
  Bashlykov, Batra, Bhargava, Bhosale, et~al.]{touvron2023llama}
Touvron, H., Martin, L., Stone, K., Albert, P., Almahairi, A., Babaei, Y.,
  Bashlykov, N., Batra, S., Bhargava, P., Bhosale, S., et~al.
\newblock Llama 2: Open foundation and fine-tuned chat models.
\newblock \emph{arXiv preprint arXiv:2307.09288}, 2023.

\bibitem[Tsai et~al.(2021{\natexlab{a}})Tsai, Bai, Morency, and
  Salakhutdinov]{DBLP:journals/corr/abs-2104-13712}
Tsai, Y.~H., Bai, S., Morency, L., and Salakhutdinov, R.
\newblock A note on connecting barlow twins with negative-sample-free
  contrastive learning.
\newblock \emph{CoRR}, abs/2104.13712, 2021{\natexlab{a}}.
\newblock URL \url{https://arxiv.org/abs/2104.13712}.

\bibitem[Tsai et~al.(2021{\natexlab{b}})Tsai, Bai, Morency, and
  Salakhutdinov]{tsai2021note}
Tsai, Y.-H.~H., Bai, S., Morency, L.-P., and Salakhutdinov, R.
\newblock A note on connecting barlow twins with negative-sample-free
  contrastive learning.
\newblock \emph{arXiv preprint arXiv:2104.13712}, 2021{\natexlab{b}}.

\bibitem[van~der Maaten \& Hinton(2008)van~der Maaten and
  Hinton]{Maaten2008VisualizingDU}
van~der Maaten, L. and Hinton, G.~E.
\newblock Visualizing data using t-sne.
\newblock \emph{Journal of Machine Learning Research}, 9:\penalty0 2579--2605,
  2008.

\bibitem[von Neumann(1932)]{john1932mathematische}
von Neumann, J.
\newblock Mathematische grundlagen der quantenmechanik, 1932.

\bibitem[Wang \& Isola(2020)Wang and Isola]{wang2020understanding}
Wang, T. and Isola, P.
\newblock Understanding contrastive representation learning through alignment
  and uniformity on the hypersphere.
\newblock In \emph{International Conference on Machine Learning}, pp.\
  9929--9939. PMLR, 2020.

\bibitem[Wang et~al.(2022)Wang, Zhang, Wang, Yang, and Lin]{wang2022chaos}
Wang, Y., Zhang, Q., Wang, Y., Yang, J., and Lin, Z.
\newblock Chaos is a ladder: A new theoretical understanding of contrastive
  learning via augmentation overlap.
\newblock \emph{arXiv preprint arXiv:2203.13457}, 2022.

\bibitem[Wen \& Li(2022)Wen and Li]{wen2022mechanism}
Wen, Z. and Li, Y.
\newblock The mechanism of prediction head in non-contrastive self-supervised
  learning.
\newblock \emph{arXiv preprint arXiv:2205.06226}, 2022.

\bibitem[Witten(2020)]{witten2020mini}
Witten, E.
\newblock A mini-introduction to information theory.
\newblock \emph{La Rivista del Nuovo Cimento}, 43\penalty0 (4):\penalty0
  187--227, 2020.

\bibitem[Wu et~al.(2018)Wu, Xiong, Yu, and Lin]{wu2018unsupervised}
Wu, Z., Xiong, Y., Yu, S.~X., and Lin, D.
\newblock Unsupervised feature learning via non-parametric instance
  discrimination.
\newblock In \emph{Proceedings of the IEEE Conference on Computer Vision and
  Pattern Recognition}, pp.\  3733--3742, 2018.

\bibitem[Ye et~al.(2019)Ye, Zhang, Yuen, and Chang]{ye2019unsupervised}
Ye, M., Zhang, X., Yuen, P.~C., and Chang, S.-F.
\newblock Unsupervised embedding learning via invariant and spreading instance
  feature.
\newblock In \emph{Proceedings of the IEEE/CVF Conference on Computer Vision
  and Pattern Recognition}, pp.\  6210--6219, 2019.

\bibitem[You et~al.(2017)You, Gitman, and
  Ginsburg]{DBLP:journals/corr/abs-1708-03888}
You, Y., Gitman, I., and Ginsburg, B.
\newblock Scaling {SGD} batch size to 32k for imagenet training.
\newblock \emph{CoRR}, abs/1708.03888, 2017.
\newblock URL \url{http://arxiv.org/abs/1708.03888}.

\bibitem[Yu et~al.(2023)Yu, Jiang, Shi, Yu, Liu, Zhang, Kwok, Li, Weller, and
  Liu]{yu2023metamath}
Yu, L., Jiang, W., Shi, H., Yu, J., Liu, Z., Zhang, Y., Kwok, J.~T., Li, Z.,
  Weller, A., and Liu, W.
\newblock Metamath: Bootstrap your own mathematical questions for large
  language models.
\newblock \emph{arXiv preprint arXiv:2309.12284}, 2023.

\bibitem[Zbontar et~al.(2021)Zbontar, Jing, Misra, LeCun, and
  Deny]{zbontar2021barlow}
Zbontar, J., Jing, L., Misra, I., LeCun, Y., and Deny, S.
\newblock Barlow twins: Self-supervised learning via redundancy reduction.
\newblock In \emph{International Conference on Machine Learning}, pp.\
  12310--12320. PMLR, 2021.

\bibitem[Zhang et~al.(2021)Zhang, Bengio, Hardt, Recht, and
  Vinyals]{zhang2021understanding}
Zhang, C., Bengio, S., Hardt, M., Recht, B., and Vinyals, O.
\newblock Understanding deep learning (still) requires rethinking
  generalization.
\newblock \emph{Communications of the ACM}, 64\penalty0 (3):\penalty0 107--115,
  2021.

\bibitem[Zhou et~al.(2022)Zhou, You, Li, Liu, Liu, Qu, and Zhu]{zhou2022all}
Zhou, J., You, C., Li, X., Liu, K., Liu, S., Qu, Q., and Zhu, Z.
\newblock Are all losses created equal: A neural collapse perspective.
\newblock \emph{arXiv preprint arXiv:2210.02192}, 2022.

\bibitem[Zhu et~al.(2021)Zhu, Ding, Zhou, Li, You, Sulam, and
  Qu]{zhu2021geometric}
Zhu, Z., Ding, T., Zhou, J., Li, X., You, C., Sulam, J., and Qu, Q.
\newblock A geometric analysis of neural collapse with unconstrained features.
\newblock \emph{Advances in Neural Information Processing Systems},
  34:\penalty0 29820--29834, 2021.

\bibitem[Zhuo et~al.(2023)Zhuo, Wang, Ma, and Wang]{zhuo2023towards}
Zhuo, Z., Wang, Y., Ma, J., and Wang, Y.
\newblock Towards a unified theoretical understanding of non-contrastive
  learning via rank differential mechanism.
\newblock In \emph{The Eleventh International Conference on Learning
  Representations}, 2023.

\bibitem[Zimmermann et~al.(2021)Zimmermann, Sharma, Schneider, Bethge, and
  Brendel]{zimmermann2021contrastive}
Zimmermann, R.~S., Sharma, Y., Schneider, S., Bethge, M., and Brendel, W.
\newblock Contrastive learning inverts the data generating process.
\newblock In \emph{International Conference on Machine Learning}, pp.\
  12979--12990. PMLR, 2021.

\end{thebibliography}
\bibliographystyle{icml2024}

\newpage
\appendix
\onecolumn

\section{Appendix for Proofs}
\label{sec:proofs}

\paragraph{Proof of Lemma~\ref{lem:psd-1}.}

\begin{proof}
Consider any non-zero matrix \( \mathbf{A} \in \mathbb{R}^{m \times n} \). We want to show that \( \mathbf{A}\mathbf{A}^\top \) is positive semi-definite.

Recall that a matrix \( \mathbf{B} \) is positive semi-definite if for all vectors \( \mathbf{x} \in \mathbb{R}^m \), it holds that \( \mathbf{x}^\top \mathbf{B} \mathbf{x} \geq 0 \). We will apply this definition to \( \mathbf{A}\mathbf{A}^\top \).

Consider any vector \( \mathbf{x} \in \mathbb{R}^m \). We compute \( \mathbf{x}^\top (\mathbf{A}\mathbf{A}^\top) \mathbf{x} \) as follows:
\begin{align*}
    \mathbf{x}^\top (\mathbf{A}\mathbf{A}^\top) \mathbf{x} &= (\mathbf{x}^\top \mathbf{A}) (\mathbf{A}^\top \mathbf{x}) \\
    &= \lVert \mathbf{A}^\top \mathbf{x} \rVert^2.
\end{align*}
The last equality holds because the expression \( (\mathbf{x}^\top \mathbf{A}) (\mathbf{A}^\top \mathbf{x}) \) represents the squared norm of the vector \( \mathbf{A}^\top \mathbf{x} \).

Since the squared norm of any vector is always non-negative, \( \lVert \mathbf{A}^\top \mathbf{x} \rVert^2 \geq 0 \) for any \( \mathbf{x} \in \mathbb{R}^m \).

Therefore, \( \mathbf{x}^\top (\mathbf{A}\mathbf{A}^\top) \mathbf{x} \geq 0 \) for all \( \mathbf{x} \in \mathbb{R}^m \), which means that \( \mathbf{A}\mathbf{A}^\top \) is positive semi-definite.

This completes the proof.
\end{proof}

\paragraph{Proof of Proposition~\ref{prop:optimal-point-matrix-kl}.}

\begin{proof}
\label{proof:optimal-point-matrix-kl}
We consider the matrix KL divergence \(\operatorname{MKL}(\mathbf{P} || \mathbf{Q})\) for positive semi-definite matrices \( \mathbf{P}, \mathbf{Q} \in \mathbb{R}^{n \times n} \). Our goal is to show that this function attains its minimum when \(\mathbf{Q} = \mathbf{P}\).

First, we calculate the gradient of \(\operatorname{MKL}(\mathbf{P} || \mathbf{Q})\) with respect to \(\mathbf{Q}\). Utilizing the properties of the matrix logarithm and trace, we find
\[
\nabla_{\mathbf{Q}} \operatorname{MKL}(\mathbf{P} || \mathbf{Q}) = -\mathbf{P} \mathbf{Q}^{-1} + \mathbf{I},
\]
where \(\mathbf{I}\) is the identity matrix.

Setting this gradient to zero, we obtain the condition for stationary points:
\[
-\mathbf{P} \mathbf{Q}^{-1} + \mathbf{I} = \mathbf{0} \implies \mathbf{P} \mathbf{Q}^{-1} = \mathbf{I}.
\]
Multiplying both sides of this equation by \(\mathbf{Q}\) yields \(\mathbf{Q} = \mathbf{P}\), indicating that \(\mathbf{Q} = \mathbf{P}\) is a stationary point of the function.

To confirm that \(\mathbf{Q} = \mathbf{P}\) is indeed a minimum, we examine the second-order conditions. The Hessian of \(\operatorname{MKL}(\mathbf{P} || \mathbf{Q})\), computed as
\[
\nabla^2_{\mathbf{Q}} \operatorname{MKL}(\mathbf{P} || \mathbf{Q}) = \mathbf{P} \mathbf{Q}^{-2},
\]
is positive semi-definite. This is because for any non-zero matrix \( \mathbf{X} \in \mathbb{R}^{n \times n} \), the expression
\[
\mathbf{X}^{\top} (\mathbf{P} \mathbf{Q}^{-2}) \mathbf{X}
\]
is non-negative, given that both \( \mathbf{P} \) and \( \mathbf{Q}^{-2} \) are positive semi-definite. Therefore, \(\operatorname{MKL}(\mathbf{P} || \mathbf{Q})\) is convex in \(\mathbf{Q}\).

Given the convexity of the function and the identification of a stationary point at \(\mathbf{Q} = \mathbf{P}\), we can conclude that this point is indeed the global minimum of the function over the domain of positive semi-definite matrices.

Hence, we conclude that
\[
\operatorname{argmin}_{\mathbf{Q} \succ 0} \operatorname{MKL}(\mathbf{P} || \mathbf{Q}) = \mathbf{P},
\]
thereby completing the proof.
\end{proof}

\paragraph{Proof of Proposition~\ref{prop:optimal-point-mce}.}

\begin{proof}
\label{proof:optimal-point-mce}
The matrix cross-entropy between two positive semi-definite matrices $\mathbf{P}$ and $\mathbf{Q}$ is defined as:
$$
\operatorname{MCE}(\mathbf{P}, \mathbf{Q})=\operatorname{tr}(-\mathbf{P} \log \mathbf{Q} + \mathbf{Q}).
$$

To find the matrix $\mathbf{Q}$ that minimizes $\operatorname{MCE}(\mathbf{P}, \mathbf{Q})$, we compute the derivative of $\operatorname{MCE}$ with respect to $\mathbf{Q}$. The derivative of the matrix cross-entropy is given by:
$$
\frac{\partial \operatorname{MCE}}{\partial \mathbf{Q}}=-\mathbf{P Q}^{-1}+\mathbf{I},
$$
where we utilized the matrix calculus result that the derivative of $\log \mathbf{Q}$ with respect to $\mathbf{Q}$ is $\mathbf{Q}^{-1}$.

Setting this derivative to zero for optimality, we get:
$$
-\mathbf{P Q}^{-1}+\mathbf{I}=\mathbf{0} \Longrightarrow \mathbf{P Q}^{-1}=\mathbf{I}.
$$
Multiplying both sides by $\mathbf{Q}$, we obtain:
$$
\mathbf{P}=\mathbf{Q}.
$$

To confirm that \(\mathbf{Q} = \mathbf{P}\) is indeed a minimum, we examine the second-order conditions, the proof is similar to Proof~\ref{proof:optimal-point-matrix-kl} for Proposition~\ref{prop:optimal-point-matrix-kl}. Therefore, we conclude that the matrix $\mathbf{Q}$ minimizing the matrix cross-entropy $\operatorname{MCE}(\mathbf{P}, \mathbf{Q})$ is $\mathbf{P}$ itself, i.e.,
$$
\operatorname{argmin}_{\mathbf{Q} \succ 0} \operatorname{MCE}(\mathbf{P}, \mathbf{Q})=\mathbf{P}.
$$
This completes the proof.
\end{proof}

\paragraph{Proof of Theorem~\ref{thm:mce-tcr-md}.}

\begin{proof}
\label{proof:mce-tcr-md}
First, begin with $\mathcal{L}_{\text{UMCE}}$ :
\[
\mathcal{L}_{\text{UMCE}}=\operatorname{MCE}\left(\frac{1}{d} \mathbf{I}_d+\lambda \mathbf{I}_d, \frac{1}{B} \mathbf{Z} \mathbf{Z}^{\top}+\lambda \mathbf{I}_d\right),
\]

Using the definition of MCE, we get:
\[
\mathcal{L}_{\text{UMCE}}=\operatorname{tr}\left(-\left(\frac{1}{d} \mathbf{I}_d+\lambda \mathbf{I}_d\right) \log \left(\frac{1}{B} \mathbf{Z} \mathbf{Z}^{\top}+\lambda \mathbf{I}_d\right)+\frac{1}{B} \mathbf{Z} \mathbf{Z}^{\top}+\lambda \mathbf{I}_d\right),
\]

Now, let us divide and multiply by $\lambda$ of the term $-\log \left(\frac{1}{B} \mathbf{Z} \mathbf{Z}^{\top}+\lambda \mathbf{I}_d\right)$:
\[
-\log \left(\frac{1}{B} \mathbf{Z} \mathbf{Z}^{\top}+\frac{\epsilon^2}{d} \mathbf{I}_d\right)=-\log \left(\lambda\left(\frac{1}{\lambda B} \mathbf{Z} \mathbf{Z}^{\top}+ \mathbf{I}_d\right)\right),
\]

Now, factor out $\lambda$:
\[
-\log \left(\lambda\left(\frac{1}{\lambda B} \mathbf{Z} \mathbf{Z}^{\top}+ \mathbf{I}_d\right)\right)=-\log (\lambda) \mathbf{I}_d-\log \left(\frac{1}{\lambda B} \mathbf{Z} \mathbf{Z}^{\top}+ \mathbf{I}_d\right),
\]

Since $\mathcal{L}_{\text{TCR}}=\frac{1}{2} \log \operatorname{det}\left(\mathbf{I}_d+\frac{d}{B \epsilon^2} \mathbf{Z} \mathbf{Z}^{\top}\right)$, we can rewrite this term in the form of $\mathcal{L}_{\text{TCR}}$.
\[
\operatorname{tr} \left(-\log \left(\frac{1}{\lambda B} \mathbf{Z} \mathbf{Z}^{\top} + \mathbf{I}_d\right)\right)= \operatorname{tr} \left(-\log \left(\mathbf{I}_d+\frac{d}{B \epsilon^2} \mathbf{Z Z}^{\top}\right) \right) = 2 \mathcal{L}_{\text{TCR}},
\]

Upon substitution, it becomes:
\[
\mathcal{L}_{\text{UMCE}}= -\operatorname{tr}\left(\left(\frac{1}{d} \mathbf{I}_d+\lambda \mathbf{I}_d\right)\left(\log (\lambda) \mathbf{I}_d\right)\right) + 2(1 + d \lambda) \mathcal{L}_{\text{TCR}}  +\operatorname{tr}\left(\frac{1}{B} \mathbf{Z} \mathbf{Z}^{\top}+\lambda \mathbf{I}_d \right),
\]

Simplifying, we get:
\[
\begin{aligned}
\mathcal{L}_{\text{UMCE}} &= -(1 + d \lambda)\log \lambda +2 (1 + d \lambda) \mathcal{L}_{\text{TCR}}+1+d \lambda\\
&= (1 + d \lambda) \left(-\log \lambda + 1 + 2 \mathcal{L}_{\text{TCR}}\right).
\end{aligned}
\]
This matches the expression given in the proposition for $\mathcal{L}_{\text{UMCE}}$.

For $\mathcal{L}_{\text{UMKL}}$, Using the definition of Matrix KL divergence, we have:
\[
\begin{aligned}
\mathcal{L}_{\text{UMKL}} &= \operatorname{MKL}\left(\left.\frac{1}{d}\mathbf{I}_d + \lambda \mathbf{I}_d \,\right | \left | \,\frac{1}{B}\mathbf{Z} \mathbf{Z}^{\top} + \lambda \mathbf{I}_d\right.\right),\\
&= \operatorname{MCE}\left(\frac{1}{d} \mathbf{I}_d+\lambda \mathbf{I}_d, \frac{1}{B} \mathbf{Z} \mathbf{Z}^{\top}+\lambda \mathbf{I}_d\right) + \operatorname{tr}\left( \mathbf{P} \log \mathbf{P} - \mathbf{P} \right),\\
\end{aligned}
\]
where $\mathbf{P}$ denotes $\frac{1}{d} \mathbf{I}_d+\lambda \mathbf{I}_d$. 

Now, we simplify $\operatorname{tr}\left( \mathbf{P} \log \mathbf{P} - \mathbf{P}\right)$. We know that \( \mathbf{P} = \frac{1}{d} \mathbf{I}_d + \lambda \mathbf{I}_d = \left( \frac{1}{d} + \lambda \right) \mathbf{I}_d \).

Since \( \mathbf{P} \) is a diagonal matrix with all diagonal entries being \( \frac{1}{d} + \lambda \), its matrix logarithm \( \log \mathbf{P} \) will also be a diagonal matrix with all diagonal entries being \( \log \left( \frac{1}{d} + \lambda \right) \).

Thus, \( \operatorname{tr}\left( \mathbf{P} \log \mathbf{P} - \mathbf{P} \right) \) can be simplified as follows:
\[
\operatorname{tr}\left( \mathbf{P} \log \mathbf{P} - \mathbf{P} \right) = \operatorname{tr}\left( \left( \frac{1}{d} + \lambda \right)\mathbf{I}_d \left(\log \left( \frac{1}{d} + \lambda \right) \mathbf{I}_d\right) - \left( \frac{1}{d} + \lambda \right)\mathbf{I}_d \right),
\]

Since the diagonal matrix $\mathbf{I}_d$ has $d$ ones along its diagonal, the trace operation essentially multiplies each term by $d$.
Therefore, we can write:
$$
\operatorname{tr}(\mathbf{P} \log \mathbf{P}-\mathbf{P})=d\left(\left(\frac{1}{d}+\lambda\right) \log \left(\frac{1}{d}+\lambda\right)-\left(\frac{1}{d}+\lambda\right)\right),
$$
Further simplifying, we get:
$$
\begin{aligned}
\operatorname{tr}(\mathbf{P} \log \mathbf{P}-\mathbf{P}) &= d\left(\frac{1}{d}+\lambda\right) \log \left(\frac{1}{d}+\lambda\right)-d\left(\frac{1}{d}+\lambda\right) \\
 &= (1+d \lambda)(\log (1+d \lambda) - \log d -1),
\end{aligned}
$$

Now, we can rewrite \( \mathcal{L}_{\text{UMKL}} \) using this result:
\[
\begin{aligned}
\mathcal{L}_{\text{UMKL}} &= \mathcal{L}_{\text{UMCE}} + \operatorname{tr}\left( \mathbf{P} \log \mathbf{P} - \mathbf{P} \right)    \\
&= \mathcal{L}_{\text{UMCE}} + (1 + d \lambda)(\log (1+d \lambda) - \log d -1)\\
&= -(1 + d \lambda)\log \lambda + 2(1 + d \lambda) \mathcal{L}_{\text{TCR}}+1+d \lambda + (1+d \lambda)(\log (1+d \lambda) - \log d -1) \\
&= -(1 + d \lambda)\log \lambda + 2(1 + d \lambda) \mathcal{L}_{\text{TCR}}+ (1+d \lambda)\log (1+d \lambda) - (1+d \lambda)\log d \\
&= (1 + d \lambda)(-\log \lambda + 2 \mathcal{L}_{\text{TCR}} + \log (1 + d \lambda) - \log d)\\
&= (1 + d \lambda)(\log \frac{1 + d \lambda}{\lambda d} + 2 \mathcal{L}_{\text{TCR}}).\\
\end{aligned}
\]

This equation represents \( \mathcal{L}_{\text{UMKL}} \) in terms of \( \mathcal{L}_{\text{TCR}} \) and other constants \( d \), \( \lambda \), and \( B \), thus fulfilling the proposition.
\end{proof}

\paragraph{Proof of Theorem~\ref{thm:minimize-tcr}.}

\begin{proof}
Here we present an alternative proof without resorting to other literature. To prove the theorem, we examine the form of the TCR loss:
\[
\mathcal{L}_{\text{TCR}} = -\frac{1}{2} \log \operatorname{det} \left( \mathbf{I}_d + \frac{d}{B \epsilon^2} \mathbf{Z} \mathbf{Z}^{\top} \right),
\]
where $\mathbf{Z} = [\boldsymbol{f}(x_1), \cdots, \boldsymbol{f}(x_B)]\in \mathbb{R}^{d\times B}$.

We note that $\mathbf{Z} \mathbf{Z}^{\top}$ is a positive semi-definite matrix, as it is the product of a matrix and its transpose. Hence, all its eigenvalues are non-negative. Let these eigenvalues be denoted by $\lambda_1, \lambda_2, \ldots, \lambda_d$. 

The determinant of $\mathbf{I}_d + \frac{d}{B \epsilon^2} \mathbf{Z} \mathbf{Z}^{\top}$ can then be expressed as the product of its eigenvalues:
\[
\operatorname{det} \left( \mathbf{I}_d + \frac{d}{B \epsilon^2} \mathbf{Z} \mathbf{Z}^{\top} \right) = \prod_{i=1}^{d} (1 + \frac{d}{B \epsilon^2} \lambda_i).
\]
Since logarithm is a monotonically increasing function, minimizing $\mathcal{L}_{\text{TCR}}$ is equivalent to maximizing the product of $(1 + \frac{d}{B \epsilon^2} \lambda_i)$ terms.

Applying the arithmetic mean-geometric mean inequality, we find that the product of the eigenvalues (and thus the determinant) is maximized when all eigenvalues are equal, i.e., $\lambda_i = \frac{B}{d}$ for all $i$. Therefore, the matrix that maximizes this determinant under the given constraints is one where all eigenvalues are $\frac{B}{d}$.

Hence, the global and unique minimizer of the TCR loss under the constraint $\| \mathbf{z}_i\|^2_2 = 1$ is achieved when $\frac{1}{B}\mathbf{Z} \mathbf{Z}^{\top}$ has eigenvalues equal to $\frac{1}{d}$, which corresponds to $\frac{1}{B}\mathbf{Z} \mathbf{Z}^{\top} = \frac{1}{d}\mathbf{I}_d$.
\end{proof}

\paragraph{Proof of Theorem~\ref{thm:covariance-matrix-effective-rank}.}

\begin{proof}
Based on the definition of effective rank presented in Section~\ref{sec:effective-rank}, a maximal effective rank of $d$ implies that the covariance matrix has $d$ non-negligible eigenvalues.

Let $\mathbf{x} = [x_1, x_2, \ldots, x_d]^\top$ be a random vector on $S^{d-1}$. The covariance matrix $\mathbf{C}(\mathbf{x})$ of $\mathbf{x}$ is defined as $\mathbb{E}[\mathbf{x}\mathbf{x}^\top] - \mathbb{E}[\mathbf{x}]\mathbb{E}[\mathbf{x}]^\top$.

The trace of $\mathbf{C}(\mathbf{x})$, which is the sum of its eigenvalues, must be at least 1. Given the maximal effective rank $d$, each of these $d$ eigenvalues must be equal (denote this common value as $\lambda$), resulting in $\mathbf{C}(\mathbf{x}) = \lambda \mathbf{I}_d$.

From above, we find that $\mathbb{E}[\mathbf{x}\mathbf{x}^\top] = \lambda \mathbf{I}_d$. Noticing that $\operatorname{tr}(\mathbf{C}(\mathbf{x}))=1-\|\mathbb{E}[\mathbf{x}]\|^2 \leq 1$ and the trace at least $1$ assumption, the trace of this matrix, which is $d\lambda$, must be equal to 1, implying $\lambda = \frac{1}{d}$.

Thus, we conclude that if the covariance matrix of $\mathbf{x}$ has the maximal possible effective rank of $d$ and its trace is at least one, then the expected value of $\mathbf{x}$ is zero, and the covariance matrix $\mathbf{C}(\mathbf{x})$ is $\frac{1}{d}\mathbf{I}_d$.
\end{proof}

\paragraph{Proof of Lemma~\ref{lem:covariance}.}
\label{sec:centering-covariance}

\begin{proof}
 To prove the lemma, we first apply the centering matrix \(\bfH_B\) to \(\bfZ_1\) and \(\bfZ_2\) as follows:
\[
\begin{aligned}
\bar{\bfZ}_1 &= \bfZ_1 \bfH_B, \\
\bar{\bfZ}_2 &= \bfZ_2 \bfH_B.
\end{aligned}
\]
These equations remove the mean of each row, effectively centering the data.

The cross-covariance matrix for the centered data \(\bar{\bfZ}_1\) and \(\bar{\bfZ}_2\) is then given by:
\[
\bfC(\bar{\bfZ}_1, \bar{\bfZ}_2) = \frac{1}{B} \bar{\bfZ}_1 \bar{\bfZ}_2^{\top}.
\]

Substituting the expressions for \(\bar{\bfZ}_1\) and \(\bar{\bfZ}_2\), we get:
\[
\bfC\left(\bfZ_1, \bfZ_2\right) = \frac{1}{B} (\bfZ_1 \bfH_B) (\bfZ_2 \bfH_B)^{\top}.
\]

Because \(\bfH_B\) is symmetric (\(\bfH_B = \bfH_B^{\top}\)) and idempotent (\(\bfH_B^2 = \bfH_B\)), this expression simplifies to:
\[
\bfC\left(\bfZ_1, \bfZ_2\right) = \frac{1}{B} \bfZ_1 \bfH_B \bfZ_2^{\top},
\]
completing the proof.
\end{proof}

\paragraph{Proof of Proposition~\ref{prop:matrix-entropy-kl-erank-relationship}.}

Recall the definition of Matrix KL divergence:
\[
\operatorname{MKL}(\mathbf{P}\, || \,\mathbf{Q}) = \operatorname{tr}(\mathbf{P} \log \mathbf{P} - \mathbf{P} \log \mathbf{Q} - \mathbf{P} + \mathbf{Q}),
\]

Substitute \(\mathbf{P} = \frac{1}{B}\mathbf{Z}\mathbf{Z}^{\top}\) and \(\mathbf{Q} = \frac{1}{d}\mathbf{I}_d\) into this:
\[
\begin{aligned}
\operatorname{MKL}\left(\left.\frac{1}{B}\mathbf{Z}\mathbf{Z}^{\top}\, \right | \left | \,\frac{1}{d}\mathbf{I}_d\right.\right) &= \operatorname{tr}\left(\frac{1}{B}\mathbf{Z}\mathbf{Z}^{\top} \log \left(\frac{1}{B}\mathbf{Z}\mathbf{Z}^{\top}\right) - \frac{1}{B}\mathbf{Z}\mathbf{Z}^{\top} \log \left(\frac{1}{d}\mathbf{I}_d\right) - \frac{1}{B}\mathbf{Z}\mathbf{Z}^{\top} + \frac{1}{d}\mathbf{I}_d\right) \\
&= \operatorname{tr}\left(\frac{1}{B}\mathbf{Z}\mathbf{Z}^{\top} \log \left(\frac{1}{B}\mathbf{Z}\mathbf{Z}^{\top}\right) + \frac{\log d}{B}\mathbf{Z}\mathbf{Z}^{\top} - \frac{1}{B}\mathbf{Z}\mathbf{Z}^{\top} + \frac{1}{d}\mathbf{I}_d\right) \\
&= -\operatorname{VNE}\left(\frac{1}{B}\mathbf{Z}\mathbf{Z}^{\top}\right) + \frac{\log d}{B}\operatorname{tr}(\mathbf{Z}\mathbf{Z}^{\top}) - \frac{1}{B}\operatorname{tr}(\mathbf{Z}\mathbf{Z}^{\top}) + \frac{1}{d}\operatorname{tr}(\mathbf{I}_d) \\
&= -\operatorname{VNE}\left(\frac{1}{B}\mathbf{Z}\mathbf{Z}^{\top}\right) + \log d - 1 + \frac{d}{d} \\
&= -\operatorname{VNE}\left(\frac{1}{B}\mathbf{Z}\mathbf{Z}^{\top}\right) + \log d,
\end{aligned}
\]

From this, we conclude that:
\[
\operatorname{VNE}\left(\frac{1}{B}\mathbf{Z}\mathbf{Z}^{\top}\right) = -\operatorname{KL}\left(\frac{1}{B}\mathbf{Z}\mathbf{Z}^{\top}\, || \,\frac{1}{d}\mathbf{I}_d\right) + \log d.
\]
\[
\operatorname{ME}\left(\frac{1}{B}\mathbf{Z}\mathbf{Z}^{\top}\right) = \operatorname{VNE}\left(\frac{1}{B}\mathbf{Z}\mathbf{Z}^{\top}\right) + \operatorname{tr}\left(\frac{1}{B} \bfZ \bfZ^{\top}\right) = \operatorname{VNE}\left(\frac{1}{B}\mathbf{Z}\mathbf{Z}^{\top}\right) + 1.
\]
The effective rank is defined as:
\[
\operatorname{erank}(\mathbf{A}) = \exp{\left\{\operatorname{H}(p_1, p_2, \ldots, p_n)\right\}},
\]
If we substitute \(\mathbf{A} = \frac{1}{B}\mathbf{Z}\mathbf{Z}^{\top}\) and given that \(\operatorname{VNE}\left(\frac{1}{B}\mathbf{Z}\mathbf{Z}^{\top}\right)\) is the entropy of the eigenvalue distribution of \(\frac{1}{B}\mathbf{Z}\mathbf{Z}^{\top}\), then we could directly relate \(\operatorname{erank}\left(\frac{1}{B}\mathbf{Z}\mathbf{Z}^{\top}\right)\) and \(\operatorname{VNE}\left(\frac{1}{B}\mathbf{Z}\mathbf{Z}^{\top}\right)\):
\[
\operatorname{erank}\left(\frac{1}{B}\mathbf{Z}\mathbf{Z}^{\top}\right) = \exp{\left\{\operatorname{VNE}\left(\frac{1}{B}\mathbf{Z}\mathbf{Z}^{\top}\right)\right\}} = \exp{\left\{\operatorname{ME}\left(\frac{1}{B}\mathbf{Z}\mathbf{Z}^{\top} \right) - 1\right\}}.
\]
Finally, we have
\[
\operatorname{erank}\left(\frac{1}{B}\mathbf{Z}\mathbf{Z}^{\top}\right) = \operatorname{exp}\left(\log d - \operatorname{MKL}\left(\left.\frac{1}{B}\mathbf{Z}\mathbf{Z}^{\top}\, \right | \left | \,\frac{1}{d}\mathbf{I}_d\right.\right) \right) = \frac{d}{\operatorname{exp}\left(\operatorname{MKL}\left(\frac{1}{B}\mathbf{Z}\mathbf{Z}^{\top}\, || \,\frac{1}{d}\mathbf{I}_d\right)\right)}
\]

\begin{theorem}[Taylor series expansion~\citep{hall2013lie}]
\label{thm:taylorhall2013}
The function
$$
\log \mathbf{A}=\sum_{m=1}^{\infty}(-1)^{m+1} \frac{(\mathbf{A}-\mathbf{I})^m}{m},
$$
is defined and continuous on the set of all $n \times n$ complex matrices $\mathbf{A}$ with $\|\mathbf{A}-\mathbf{I}\|<1$.
For all $\mathbf{A}$ with $\|\mathbf{A}-\mathbf{I}\|<1$,
$$
e^{\log \mathbf{A}}=\mathbf{A}.
$$
For all $\mathbf{X}$ with $\|\mathbf{X}\|_{F}<\log 2,\left\|e^{\mathbf{X}}-\mathbf{I}\right\|<1$ and
$$
\log e^{\mathbf{X}}=\mathbf{X}.
$$
\end{theorem}

\section{Details on Experiments}
\label{sec:more-experiment-details}

\subsection{More on Loss Functions}

Now we take a closer look at the loss function:
\begin{equation}
\label{eq:matrix-ssl-more}
\begin{aligned}
\mathcal{L}_{\text{Matrix-SSL}} &= \mathcal{L}_{\text{Matrix-Uniformity}} + \mathcal{L}_{\text{Matrix-Alignment}(\gamma)}\\
&= \operatorname{MCE}\left(\frac{1}{d}\mathbf{I}_d, \bfC\left(\bfZ_1, \bfZ_2\right)\right) -\operatorname{tr} \left(\bfC\left(\bfZ_1, \bfZ_2\right)\right)
+ \gamma \cdot \operatorname{MCE}\left(\mathbf{C}(\mathbf{Z}_1, \mathbf{Z}_1), \mathbf{C}(\mathbf{Z}_2, \mathbf{Z}_2)\right)\\
&= -\operatorname{tr}\left(\left(\frac{1}{d}\mathbf{I}_d\right) \log\left(\bfC\left(\bfZ_1, \bfZ_2\right)\right)\right) -\gamma \cdot \operatorname{tr}\left(\mathbf{C}(\mathbf{Z}_1, \mathbf{Z}_1) \log \left(\mathbf{C}(\mathbf{Z}_2, \mathbf{Z}_2)\right)\right)+ \gamma \cdot \operatorname{tr}\left(\mathbf{C}(\mathbf{Z}_2, \mathbf{Z}_2)\right) + \text{Const. }\\
\end{aligned}
\end{equation}
\paragraph{Employing matrix KL divergence.} As we previously introduced in Section~\ref{sec:implementation-linear-evaluation}, applying the stop gradient technique to the first branch $\mathbf{Z}_1$, as utilized in SimSiam~\cite{Hua2021SimSiam}, renders the third term $\operatorname{ME}(\operatorname{C}(\mathbf{Z}_1, \mathbf{Z}_1))$ a constant in the Matrix-Alignment-KL loss, as delineated in Equation~\ref{eq:mkl-alignment-more}.
\begin{equation}
\label{eq:mkl-alignment-more}
\begin{aligned}
\mathcal{L}_{\text{Matrix-Alignment-KL}} &= -\operatorname{tr} \left(\mathbf{C}(\bfZ_1, \bfZ_2)\right) +\gamma \cdot \operatorname{MKL}\left(\mathbf{C}(\mathbf{Z}_1, \mathbf{Z}_1) || \mathbf{C}(\mathbf{Z}_2, \mathbf{Z}_2)\right)\\
&= -\operatorname{tr} \left(\bfC\left(\bfZ_1, \bfZ_2\right)\right) + \gamma \cdot \operatorname{MCE}\left(\mathbf{C}(\mathbf{Z}_1, \mathbf{Z}_1), \mathbf{C}(\mathbf{Z}_2, \mathbf{Z}_2)\right) - \gamma \cdot \operatorname{ME}(\mathbf{C}(\mathbf{Z}_1, \mathbf{Z}_1)).
\end{aligned}
\end{equation}
\begin{equation}
\label{eq:matrix-ssl-more-kl}
\begin{aligned}
\mathcal{L}_{\text{Matrix-SSL-KL}} &= \mathcal{L}_{\text{Matrix-Uniformity-KL}} + \mathcal{L}_{\text{Matrix-Alignment-KL}(\gamma)}\\
&= \operatorname{MKL}\left(\frac{1}{d}\mathbf{I}_d  || \bfC\left(\bfZ_1, \bfZ_2\right)\right) -\operatorname{tr} \left(\bfC\left(\bfZ_1, \bfZ_2\right)\right)
+ \gamma \cdot \operatorname{MKL}\left(\mathbf{C}(\mathbf{Z}_1, \mathbf{Z}_1) || \mathbf{C}(\mathbf{Z}_2, \mathbf{Z}_2)\right)\\
&= \operatorname{MKL}\left(\frac{1}{d}\mathbf{I}_d || \bfC\left(\bfZ_1, \bfZ_2\right)\right) -\operatorname{tr} \left(\bfC\left(\bfZ_1, \bfZ_2\right)\right)
+ \gamma \cdot \operatorname{MCE}\left(\mathbf{C}(\mathbf{Z}_1, \mathbf{Z}_1), \mathbf{C}(\mathbf{Z}_2, \mathbf{Z}_2)\right) -\gamma \cdot \operatorname{ME}(\mathbf{C}(\mathbf{Z}_1, \mathbf{Z}_1)) \\
&= \operatorname{MCE}\left(\frac{1}{d}\mathbf{I}_d, \bfC\left(\bfZ_1, \bfZ_2\right)\right) - \operatorname{ME}\left(\frac{1}{d}\mathbf{I}_d\right) -\operatorname{tr} \left(\bfC\left(\bfZ_1, \bfZ_2\right)\right)\\
&\quad + \gamma \cdot \operatorname{MCE}\left(\mathbf{C}(\mathbf{Z}_1, \mathbf{Z}_1), \mathbf{C}(\mathbf{Z}_2, \mathbf{Z}_2)\right) -\gamma \cdot \operatorname{ME}(\mathbf{C}(\mathbf{Z}_1, \mathbf{Z}_1))\\
&= -\operatorname{tr}\left(\left(\frac{1}{d}\mathbf{I}_d\right) \log\left(\bfC\left(\bfZ_1, \bfZ_2\right)\right)\right) + \operatorname{tr}\left(\bfC\left(\bfZ_1, \bfZ_2\right)\right) + \text{Const. } - \operatorname{tr}\left(\bfC\left(\bfZ_1, \bfZ_2\right)\right)\\
&\quad - \gamma \cdot \operatorname{tr}\left(\mathbf{C}(\mathbf{Z}_1, \mathbf{Z}_1) \log \left(\mathbf{C}(\mathbf{Z}_2, \mathbf{Z}_2)\right)\right) + \gamma \cdot \operatorname{tr}\left(\mathbf{C}(\mathbf{Z}_2, \mathbf{Z}_2)\right) -\gamma \cdot \operatorname{ME}(\mathbf{C}(\mathbf{Z}_1, \mathbf{Z}_1))\\
&= -\operatorname{tr}\left(\left(\frac{1}{d}\mathbf{I}_d\right) \log\left(\bfC\left(\bfZ_1, \bfZ_2\right)\right)\right) -\gamma \cdot \operatorname{tr}\left(\mathbf{C}(\mathbf{Z}_1, \mathbf{Z}_1) \log \left(\mathbf{C}(\mathbf{Z}_2, \mathbf{Z}_2)\right)\right)\\
&\quad  + \gamma \cdot \operatorname{tr}\left(\mathbf{C}(\mathbf{Z}_2, \mathbf{Z}_2)\right) +\gamma \cdot \operatorname{ME}(\mathbf{C}(\mathbf{Z}_1, \mathbf{Z}_1)) + \text{Const. } \qquad \left(\xRightarrow[\text{}]{\text{Stop Gradient on }\mathbf{Z}_1}\right)\\
&=-\operatorname{tr}\left(\left(\frac{1}{d}\mathbf{I}_d\right) \log\left(\bfC\left(\bfZ_1, \bfZ_2\right)\right)\right) -\gamma \cdot \operatorname{tr}\left(\mathbf{C}(\mathbf{Z}_1, \mathbf{Z}_1) \log \left(\mathbf{C}(\mathbf{Z}_2, \mathbf{Z}_2)\right)\right) + \gamma \cdot \operatorname{tr}\left(\mathbf{C}(\mathbf{Z}_2, \mathbf{Z}_2)\right) + \text{Const.}\\
\end{aligned}
\end{equation}

From Equation~\ref{eq:matrix-ssl-more} and \ref{eq:matrix-ssl-more-kl}, we find that they are essentially the same loss function when the stop gradient is performed.

\section{Neural Collapse and Dimensional Collapse}
\label{sec:neural-collapse}

 Feature representations acquired through a deep neural network employing a cross-entropy (CE) loss optimized by stochastic gradient descent, are capable of attaining zero loss~\citep{du2018gradient} with arbitrary label assignments~\citep{zhang2021understanding}. A phenomenon which known as neural collapse (NC)~\citep{papyan2020prevalence} is observed when training of the neural network continues beyond zero loss with CE. \citet{galanti2021role} demonstrate that the NC phenomenon can facilitate some transfer learning tasks. However, potential concerns associated with neural collapse exist, as \citet{ma2023do} posit that the total within-class features collapse may not be ideal for fine-grained classification tasks.
 
 The NC phenomenon embodies the following characteristics~\citep{zhu2021geometric}:

\begin{itemize}
    \item Variability collapse: The intra-class variability of the final layer's features collapse to zero, signifying that all the features of a single class concentrate on the mean of these features for each class respectively.
    \item Convergence to Simplex ETF: Once centered at their global mean, the class-means are simultaneously linearly separable and maximally distant on a hypersphere. This results in the class-means forming a simplex equiangular tight frame (ETF), a symmetrical structure determined by a set of points on a hypersphere that is maximally distant and equiangular to each other.
    \item Convergence to self-duality: The linear classifiers, existing in the dual vector space of the class-means, converge to their respective class-mean and also construct a simplex ETF.
    \item Simplification to Nearest Class-Center (NCC): The linear classifiers behaviors similarly to the nearest class-mean decision rule.
\end{itemize}

Here we present the definition of standard $K$-Simplex ETF and general $K$-Simplex ETF~\citep{papyan2020prevalence}.

\begin{definition}[$K$-Simplex ETF]
\label{def:simplex-etf}
A standard Simplex ETF is characterized as a set of points in $\mathbb{R}^K$, defined by the columns of
$$
\mathbf{M}=\sqrt{\frac{K}{K-1}}\left(\mathbf{I}_K-\frac{1}{K} \mathbf{1}_K \mathbf{1}_K^{\top}\right),
$$
where $\mathbf{I}_K \in \mathbb{R}^{K \times K}$ is the identity matrix, and $\mathbf{1}_K \in \mathbb{R}^K$ represents a all-one vector. Consequently, we also obtain

$$
\mathbf{M}^{\top} \mathbf{M}=\mathbf{M} \mathbf{M}^{\top}=\frac{K}{K-1}\left(\mathbf{I}_K-\frac{1}{K} \mathbf{1}_K \mathbf{1}_K^{\top}\right).
$$
\end{definition}

\begin{definition}[General $K$-Simplex ETF]
\label{def:general-simplex-etf}
A general Simplex ETF is characterized as a set of points in $\mathbb{R}^K$, defined by the columns of
$$
\tilde{\mathbf{M}}=\alpha \mathbf{U} \mathbf{M},
$$
where $\alpha \in \R_+$ is a scale factor, and $\mathbf{U} \in \R^{p \times K}$ ($p \geq K$) is a partial orthogonal matrix $\mathbf{U}^{\top} \mathbf{U} = \mathbf{I}$.
\end{definition}

\citet{zhu2021geometric} further studied the problem using an unconstrained feature model that separates the topmost layers from the classifier of the neural network. They established that the conventional cross-entropy loss with weight decay presents a benign global landscape, where the only global minimizers are the Simplex ETFs and all other critical points are strict saddles exhibiting negative curvature directions.

The study was later extended~\citep{zhou2022all}, demonstrating through a global solution and landscape analysis that a wide range of loss functions, including commonly used label smoothing (LS) and focal loss (FL), display Neural Collapse. Therefore, all pertinent losses (i.e., CE, LS, FL, MSE) yield comparable features on training data.

\section{Measuring Dimensional Collapse}
\label{sec:measuring-dimensional-collapse}

\begin{figure*}[htb]
\centering
\subfigure[SimCLR]{
\includegraphics[width=0.225\textwidth]{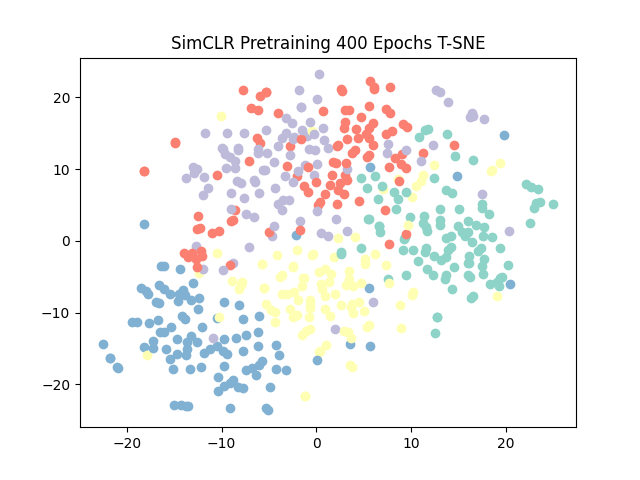}
\label{fig:tsne-simclr-5}
}
\subfigure[BYOL]{
\includegraphics[width=0.225\textwidth]{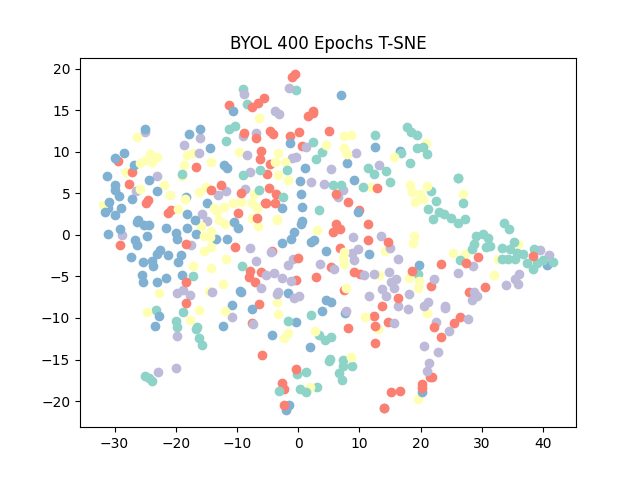}
\label{fig:tsne-byol-5}
}
\subfigure[Barlow Twins]{
\includegraphics[width=0.225\textwidth]{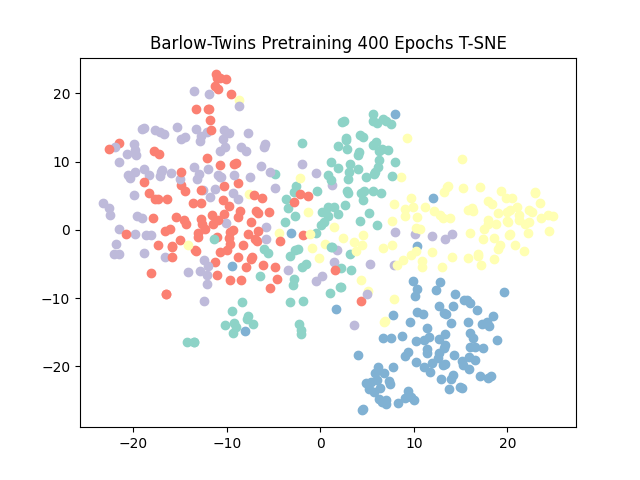}
\label{fig:tsne-barlowtwins-5}
}
\subfigure[SimSiam (collapsed w/o stop gradient)]{
\includegraphics[width=0.225\textwidth]{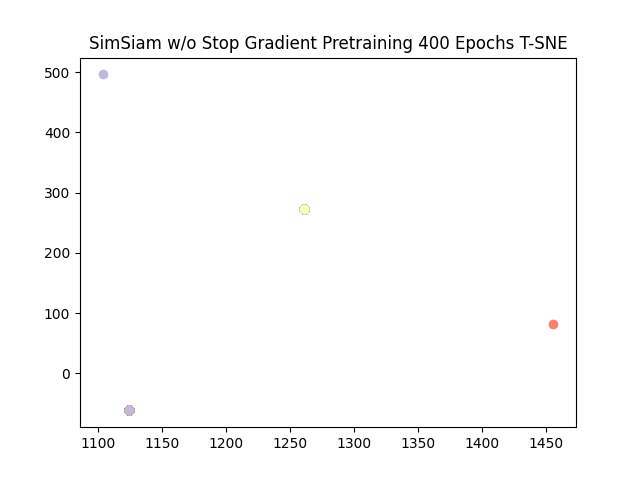}
\label{fig:tsne-collapse-5}
}
\caption{Visualization of feature representation for images in 5 different classes from CIFAR-100 dataset via t-SNE of various self-supervised learning methods. We find that SimCLR has larger inter-class variability than others, as the clusters seem more separable. For illustration, we also introduce a collapsed representation via SimSiam without stop gradient operation.}
\label{fig:t-sne-5}
\end{figure*}

\citet{papyan2020prevalence} discuss the fascinating occurrence of neural collapse during the training of a supervised neural network utilizing cross-entropy loss for classification tasks that result in an intra-class collapse. Contrastive learning has effects of dimensional collapse due to its spectral clustering nature \citep{tan2023contrastive}. As dimension-contrastive learning can be seen as pursuing uniformity, we are also interested in discovering the relationship between dimension-contrastive learning and dimensional collapse. 

Figure~\ref{fig:t-sne-5} illustrates that the non-contrastive method, Barlow Twins, exhibits greater intra-class variability than the contrastive method, SimCLR. However, for larger samples and classes (e.g., Figure~\ref{fig:t-sne-10} in Appendix~\ref{sec:neural-collapse}), this observation is qualitative explicit. To quantify this observation, we propose the introduction of metrics involving class-specific information to quantify dimensional collapse. These measures may enhance our understanding of the differences among supervised learning, contrastive, and non-contrastive SSL.

Assuming a total of $K$ classes and $n$ labeled samples $\{ x_i,y_i \}^{n}_{i=1}$, denote the number of samples in each class $c$ as $n_c$, i.e., $n_c =  |\{ i \mid y_i =c\}| $. We define the \textit{intra-class effective rank} and \textit{inter-class effective rank} as follows.

\begin{definition}[Intra-class effective rank]
Denote the class-mean vector of each class $c$ as $\mathbf{\mu}_c = \frac{1}{n_c}\sum\limits_{y_i =c} \mathbf{f}(\mathbf{x}_i)$, and denote $
\mathbf{C}(\mathbf{f}(x)\mid y)) = \frac{1}{n_y}\sum\limits_{y_i = y} (\mathbf{f}(x_i)-\mathbf{\mu}_y)(\mathbf{f}(x_i)-\mathbf{\mu}_y)^{\top}.$ We define \textit{intra-class effective rank} (intra-class erank) as 
\begin{equation}
\operatorname{erank_{intra-class}} \triangleq \frac{1}{K}\sum_{y \in [K]} \operatorname{erank}(\mathbf{C}(\mathbf{f}(\mathbf{x})\mid y))),    
\end{equation}

which can be viewed as an empirical approximation of $\mathbb{E}_{y \in [K]}\left[\operatorname{erank}(\mathbf{C}(\mathbf{f}(x)\mid y))\right]$, where $x$ is drawn from $p_{\text{data}}$.
\end{definition}

\begin{definition}[Inter-class effective rank]
Denote global mean of representation as $\mathbf{\mu}_G = \frac{1}{n}\sum_{i \in [n]} \mathbf{f}(x_i)$, then we define \textit{inter-class effective rank} (inter-class erank) as the effective rank of the covariance matrix of all $C$ class-mean vectors,
\begin{equation}
\operatorname{erank_{inter-class}} \triangleq
\operatorname{erank}[\frac{1}{K} \sum_{i \in [K]} (\mathbf{\mu}_i-\mathbf{\mu}_G)(\mathbf{\mu}_i-\mathbf{\mu}_G)^{\top}].  
\end{equation}

When class are balanced, intra-class erank is approximately $\operatorname{erank}(\mathbf{C}_{y \in [K]}(E[\mathbf{f}(x) \mid y]))$, where $x$ is drawn from $p_{\text{data}}$.    
\end{definition}

\paragraph{Remark.} These two metrics can be interpreted as an effective rank factorization of the two terms in the total covariance theorem.

\begin{wrapfigure}{r}{0.5\textwidth}
  \vspace{-7.7mm}
   \centering
   \begin{minipage}{0.5\textwidth}
    \centering
\subfigure[Intra-class erank on test dataset]{
\includegraphics[width=0.45\textwidth]{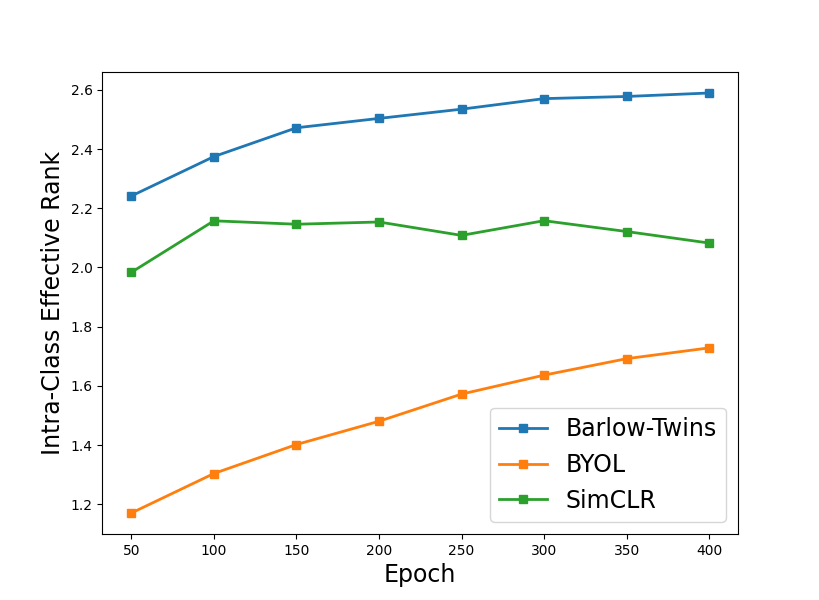}
\label{fig:test-inter-class-er}
}
\subfigure[Inter-class erank on test dataset]{
\includegraphics[width=0.45\textwidth]{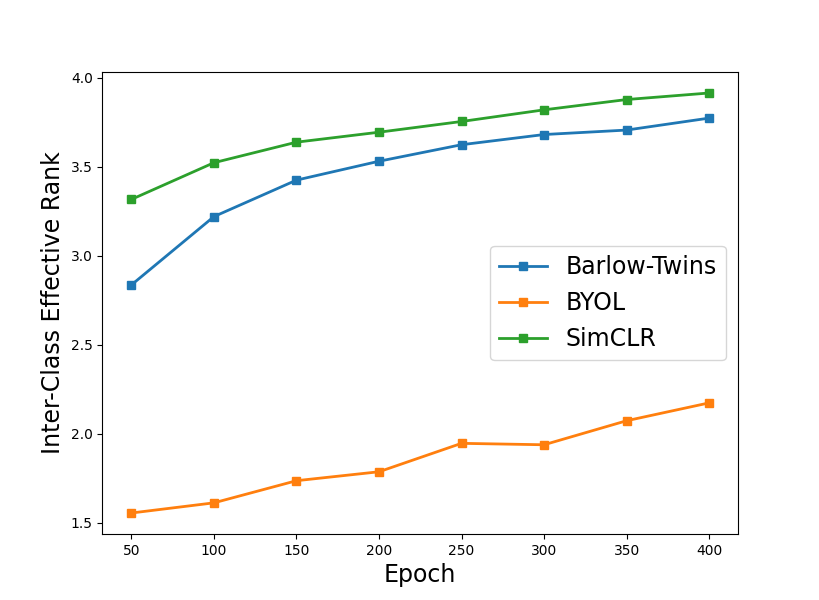}
\label{fig:test-class-means-er}
}
\caption{Intra-class effective rank and inter-class effective rank. It is obvious that intra-class effective rank continues to grow for BYOL or Barlow Twins, but not for SimCLR.}
\label{fig:inter-class-class-means-er}
 \end{minipage}
 \vspace{-1mm}
\end{wrapfigure}

From illustrative examples shown in Figure ~\ref{fig:inter-class-class-means-er}, we observe that SimCLR, as a contrastive method, exhibits a consistent decrease in intra-class effective rank during training. This empirical evidence corroborates the spectral clustering interpretation of contrastive learning. On the contrary, non-contrastive methods like BYOL and Barlow Twins, owing to the inherent property of kernel-uniformity loss (and its low-order Taylor approximations) tending towards a uniform distribution, exhibit larger intra-class effective ranks that continue to increase during training. Regarding the inter-class effective rank, a metric for global class-means effective rank, all three methods show a consistent increase.

We now present some theoretical properties of effective rank and its connections to an equiangular tight frame (ETF). The following theorem suggests that a larger effective rank of the Gram matrix is beneficial for expressiveness.

\begin{theorem}[Maximize effective rank forms a equiangular tight frame (ETF)]
\label{thm:effective-rank-etf}
For $K$ vectors $\mathbf{z}_i$ ($1 \leq i \leq K$), each lying on $S^{d-1}$. Assuming the latent dimension $d$ satisfies $d \geq K$ and the mean of $\mathbf{z}_i$ is $\mathbf{0}$, denote $\mathbf{Z} = [\mathbf{z}_1,\cdots , \mathbf{z}_K]$. If the Gram matrix $\mathbf{Z}^{\top}\mathbf{Z}$ has an effective rank of $K-1$, it implies the existence of an equiangular tight frame (ETF) in the orthonormal span of $\mathbf{z}_i$. Conversely, the Gram matrix of any ETF has an effective rank of $K-1$.
\end{theorem}

\begin{proof}
Since the mean vector is $\mathbf{0}$, the Gram matrix can have an effective rank of at most $K-1$. By Property $1$ in \cite{roy2007effective}, we deduce that the Gram matrix $\mathbf{Z}^{\top}\mathbf{Z}$ has $K-1$ equal eigenvalues and one eigenvalue equal to $0$.

The trace of the Gram matrix equals $K$ because $\mathbf{z}_i$ lies on $S^{d-1}$. Hence, the Gram matrix has $K-1$ equal eigenvalues of $\frac{K}{K-1}$ and one eigenvalue of $0$. Therefore, the Gram matrix shares the same eigenvalues (spectrum) as $\frac{K}{K-1} \mathbf{H}_K$, where $\mathbf{H}_K$ is the centering matrix $\bfI_K - \frac{1}{K} \mathbf{1_K 1_K}^{\top}$.

Given the orthonormal standard form, there exists an orthonormal matrix $\mathbf{Q} \in \mathbb{R}^{K \times K}$ such that $\mathbf{Q}^{\top}(\mathbf{Z}^{\top}\mathbf{Z})\mathbf{Q}=\frac{K}{K-1} \mathbf{H}_K$. According to Lemma $11$ in \citet{papyan2020prevalence}, $\mathbf{Z}\mathbf{Q}$ constitutes an ETF. As $\mathbf{Z}\mathbf{Q}$ directly represents the orthonormal span of $\mathbf{Z}$'s column space, the conclusion follows.
\end{proof}

Gram matrix plays a key role in connecting our metric with Section~\ref{sec:rank}, i.e., understanding the rank-increasing phenomenon.

\begin{theorem}
\label{thm:eff-2}
The effective rank of the total sample Gram matrix can be effectively estimated by batch.  
\end{theorem}

\begin{proof}
Note scaling does not change effective rank. Change the order of $\mathbf{Z}^{\top}\mathbf{Z}$ to $\mathbf{Z}\mathbf{Z}^{\top}$, then can rewrite self-correlation as the empirical estimation of expected self-correlation by samples in a batch. This explains the estimation given by~\citet{zhuo2023towards}. 
\end{proof}

Interestingly, the following theorem connects our metrics with the Gram matrix.  

\begin{theorem}
Assuming the dataset is class-balanced and the global mean is $0$, then the effective rank of the covariance matrix of all $K$ class-mean vectors is exactly the same as the effective rank of the Gram matrix.    
\end{theorem}

\begin{proof}
As $\mathbf{Z}\mathbf{Z}^{\top}$ and $\mathbf{Z}^{\top}\mathbf{Z}$ have the same non-zero eigenvalues, thus having the same effective rank.  
\end{proof}

\subsection{Experiments on Dimensional Collapse}

\begin{figure*}[htbp]
\centering
\subfigure[SimCLR]{
\includegraphics[width=0.225\textwidth]{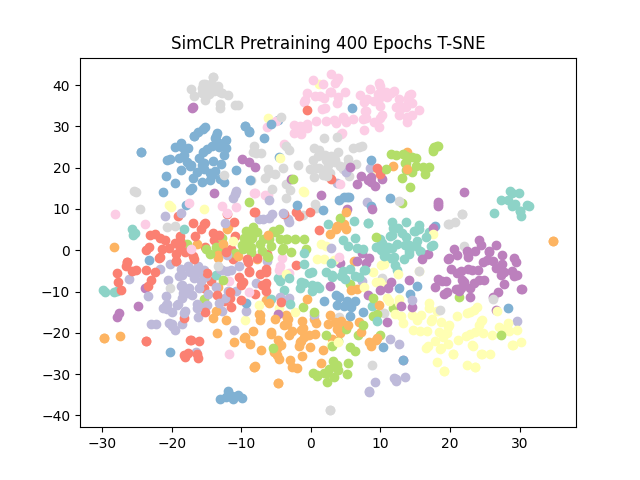}
\label{fig:tsne-simclr-10}
}
\subfigure[BYOL]{
\includegraphics[width=0.225\textwidth]{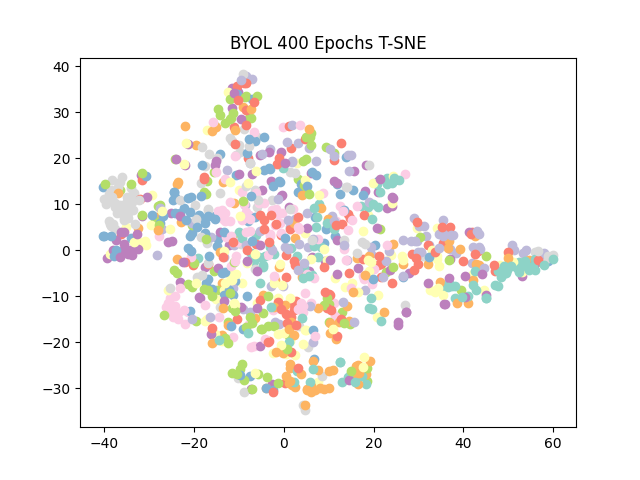}
\label{fig:tsne-byol-10}
}
\subfigure[Barlow Twins]{
\includegraphics[width=0.225\textwidth]{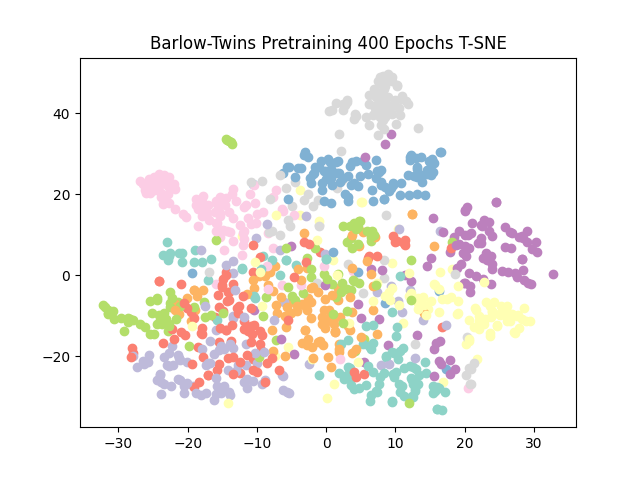}
\label{fig:tsne-barlowtwins-10}
}
\subfigure[SimSiam (collapsed w/o stop gradient)]{
\includegraphics[width=0.225\textwidth]{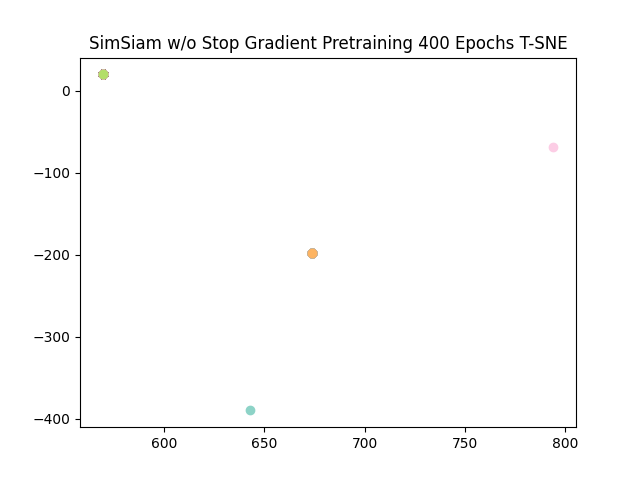}
\label{fig:tsne-collapse-10}
}
\caption{Visualization of feature representation for images in 10 different classes from CIFAR-100 dataset via t-SNE of various self-supervised learning methods. We find that in many categories, it is difficult to distinguish between two non-contrastive methods (BYOL, Barlow Twins) and contrastive method (SimCLR) by t-SNE.}
\label{fig:t-sne-10}
\end{figure*}

We measure dimensional collapse on various self-supervised learning methods, including SimCLR~\citep{chen2020simple}, BYOL~\citep{grill2020bootstrap}, Barlow Twins~\citep{zbontar2021barlow} and SimSiam~\citep{chen2021exploring} with or without stop gradient. We reproduce the above methods on the self-supervised learning task of CIFAR100~\citep{krizhevsky2009learning} dataset, using the open source implementations ~\citep{DBLP:journals/corr/abs-2104-13712, Hua2021SimSiam} of the above methods tuned for CIFAR. After pre-training, we use the saved checkpoints to evaluate the results of these methods on different metrics. 

We calculate the intra-class and inter-class effective rank directly by definition, while for MCE, we shuffle the testing dataset, import the data with 512 batch size, and finally output the average metrics of all batches.

We perform t-SNE~\citep{Maaten2008VisualizingDU} visualization on the last checkpoint of each method with the help of scikit-learn~\citep{pedregosa2011scikit}. We use the default t-SNE~\citep{Maaten2008VisualizingDU} parameter of scikit-learn~\citep{pedregosa2011scikit} and select the first $5$ or $10$ categories from $100$ categories in CIFAR-100~\citep{krizhevsky2009learning} for visualization.

\end{document}